\documentclass{article}

\usepackage[main,final]{neurips_2026}
\makeatletter
\renewcommand{\@noticestring}{}
\makeatother

\usepackage[utf8]{inputenc}
\usepackage[T1]{fontenc}
\usepackage{amsmath,amssymb,amsthm}
\usepackage{amsfonts}
\usepackage{mathtools}
\usepackage{graphicx}
\usepackage{booktabs}
\usepackage{multirow}
\usepackage{nicefrac}
\usepackage{microtype}
\usepackage{enumitem}
\usepackage[ruled,vlined]{algorithm2e}
\usepackage{caption}
\usepackage{subcaption}
\usepackage[table]{xcolor}
\usepackage{url}
\usepackage{hyperref}

\hypersetup{
  colorlinks=true,
  linkcolor=black,
  filecolor=black,
  citecolor=black,
  urlcolor=blue,
  pdfborderstyle={/S/U/W 1}
}

\newcommand{\R}{\mathbb{R}}
\newcommand{\bH}{\mathbf{H}}

\newcommand{\bA}{\mathbf{A}}
\newcommand{\bD}{\mathbf{D}}
\newcommand{\bI}{\mathbf{I}}
\newcommand{\Atilde}{\tilde{\mathbf{A}}}
\newcommand{\Ahat}{\hat{\mathbf{A}}}
\newcommand{\ECI}{\mathrm{ECI}}
\newcommand{\MAD}{\mathrm{MAD}}
\newcommand{\intra}{\mathrm{intra}}
\newcommand{\inter}{\mathrm{inter}}
\newcommand{\G}{\mathcal{G}}
\newcommand{\bh}{\mathbf{h}}

\newtheorem{theorem}{Theorem}
\newtheorem{proposition}{Proposition}
\newtheorem{lemma}{Lemma}

\newtheorem{definition}{Definition}

\definecolor{darkgreen}{rgb}{0.0,0.5,0.0}
\newcommand{\gain}[1]{%
  \textcolor{darkgreen}{\footnotesize$\uparrow$#1}%
}
\newcommand{\loss}[1]{%
  \textcolor{red}{\footnotesize$\downarrow$#1}%
}

\title{Not Just Oversmoothing: Detecting the Echo Chamber Effect in Graph Neural Networks}

\author{
  Asela Hevapathige\\
  Department of Mechanical Engineering\\
  University of Melbourne\\
  Melbourne, Australia\\
  \texttt{asela.hevapathige@unimelb.edu.au}
  \And
  Ahad N. Zehmakan\\
  School of Computing\\
  Australian National University\\
  Canberra, Australia\\
  \texttt{ahadn.zehmakan@anu.edu.au}
  \AND
  Asiri Wijesinghe\\
  Data61, CSIRO\\
  Canberra, Australia\\
  \texttt{asiriwijesinghe.wijesinghe@data61.csiro.au}
  \And
  Saman Halgamuge\\
  Department of Mechanical Engineering\\
  University of Melbourne\\
  Melbourne, Australia\\
  \texttt{saman@unimelb.edu.au}
}

\begin{document}

\maketitle

\begin{abstract}
Oversmoothing is a well-known failure mode of Graph Neural Networks (GNNs).
However, most existing diagnostics rely on global aggregation measures that
fail to capture the heterogeneous dynamics of message passing. Real-world
graphs exhibit pronounced community structure, and message passing operates
on two timescales, with representations collapsing rapidly within communities
and slowly across them. This creates a critical gap in which intra-community
representations can become indistinguishable while inter-community separation
persists, a failure mode that we refer to as the \emph{Echo Chamber Effect}.
To quantify this effect, we introduce the Echo Chamber Index (ECI), which
stratifies pairwise distances by community membership and reveals when global
energy diminishes while inter-community separation persists. ECI further
shows that feature-retention mechanisms can preserve the echo chamber under
the conditions of our theoretical analysis. The consequences depend on label
structure: when communities align with classes, the echo chamber can sharpen
node classification, whereas when they do not, the same collapse makes
classification provably harder. Motivated by this analysis, we propose
Community-Aware Split Propagation (CASP), a lightweight plugin that decouples
intra- and inter-community aggregation and learns their balance from label
structure. CASP improves diverse backbone GNNs across most evaluated homophilic and heterophilic
settings. %Our source code is available at: \url{https://anonymous.4open.science/r/CASP-3059/}.
\end{abstract}

\section{Introduction}

Graph neural networks (GNNs) generate node representations by iteratively aggregating features from neighboring nodes across multiple layers~\cite{gilmer2017neural,kipf2017semi}. As the depth of the network increases, this repeated aggregation tends to drive the representations toward a common vector, a phenomenon known as oversmoothing~\cite{rusch2023survey}. Two metrics dominate its diagnosis:
Dirichlet energy~\cite{cai2020note,bison2025analysis}, which sums squared
representation differences over edges, and Mean Average Distance
(MAD)~\cite{chen2020measuring}, which averages pairwise feature
distances across nodes. Both metrics reduce the representation to a single global scalar, an
aggregation that implicitly treats all regions of the graph as
equivalent.

Real-world networks are rarely uniform. Social, citation, and
biological networks exhibit pronounced
modularity~\cite{girvan2002community,alcala2021modularity,radicchi2011citation},
and on such graphs, message passing operates on two timescales: rapid
homogenisation within communities and slow mixing across them. Dirichlet energy and MAD aggregate over all edges and all node pairs
respectively, treating intra- and inter-community contributions
identically. Thus, they may indicate oversmoothing alleviation when meaningful inter-community 
separation persists, while intra-community representations become indistinguishable.
We refer to this state as the \emph{echo chamber effect}. On modular graphs
this region persists across GNNs well beyond the depths
typically deployed, which is a blind spot for current diagnostics.

This blind spot has consequences for both diagnosis and design.
Oversmoothing diagnostics routinely guide depth choices and motivate
design decisions such as residual or skip-connection
architectures~\cite{zhang2023comprehensive,kelesis2025analyzing}. When
a diagnostic fails to detect the echo chamber, these choices that rely
on it may be misled; consequently, architectural mechanisms may be incorrectly credited
with preventing a collapse that has already occurred. To make the echo chamber
region detectable, we introduce the \emph{Echo Chamber Index} (ECI),
defined as the ratio of cross-community to within-community pairwise
distance. ECI separates three regions of message passing on modular
graphs: \emph{uniform mixing}, in which the two distance scales are
comparable; an \emph{intermediate phase}, in which within-community
representations have collapsed but communities remain distinguishable;
and \emph{global oversmoothing}, in which all pairwise distinctions
vanish. As shown in Figure~\ref{fig:layer-metrics}, Dirichlet energy
and MAD decay monotonically across architectures, while ECI traverses
all three regions for a classic GNN such as GCN~\cite{kipf2017semi} and remains elevated for
oversmoothing-alleviation methods such as APPNP~\cite{gasteiger2019predict},
GCNII~\cite{chen2020simple}, and GPR-GNN~\cite{chien2021adaptive}.

\begin{figure}[t]
  \centering
  \includegraphics[width=\linewidth]{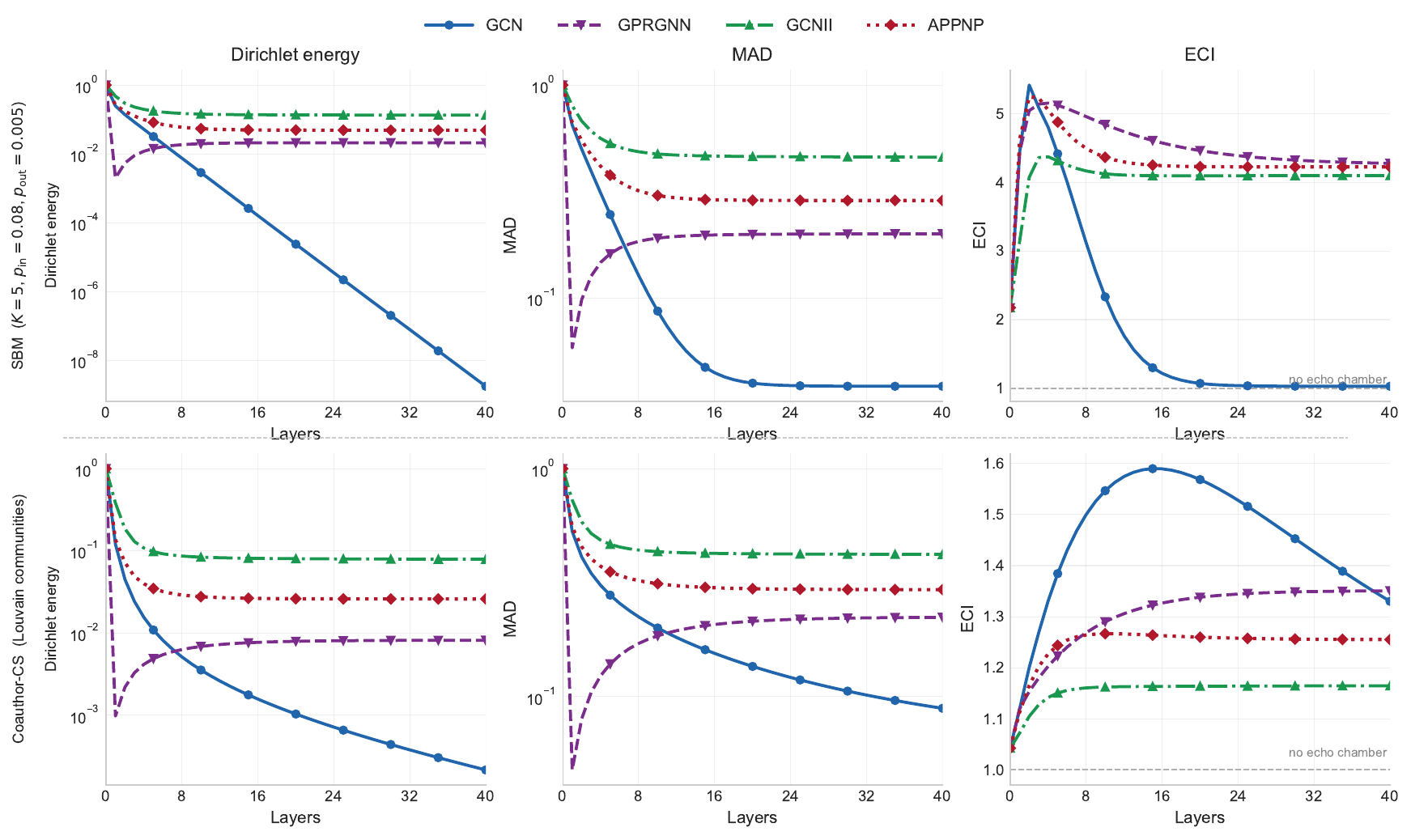}
  \caption{Layer-wise metrics on an SBM~\cite{holland1983stochastic} (1500 nodes, 5
  communities, intra-community edge probability 0.08, inter-community
  0.005) and Coauthor-CS with Louvain~\cite{blondel2008fast} communities.}
  \label{fig:layer-metrics}
\end{figure}

The echo chamber has task-specific consequences. We illustrate this
through node classification, where its effect depends on how
communities align with labels. Under homophily, communities coincide
with classes, so intra-community collapse compresses within-class
variance and improves linear separability. Under heterophily, the same
collapse forces same-community nodes of different classes into nearly
identical representations, degrading classification. We formalise both
directions and propose \emph{Community-Aware Split Propagation} (CASP),
a plugin that grants existing GNNs explicit control over the balance between
intra- and inter-community message passing. CASP calibrates the
intra/inter-community propagation balance directly from the graph's label
structure, amplifying the echo chamber under homophily and suppressing it
under heterophily. 

\begin{figure}[t]
    \centering
    \includegraphics[width=\linewidth]{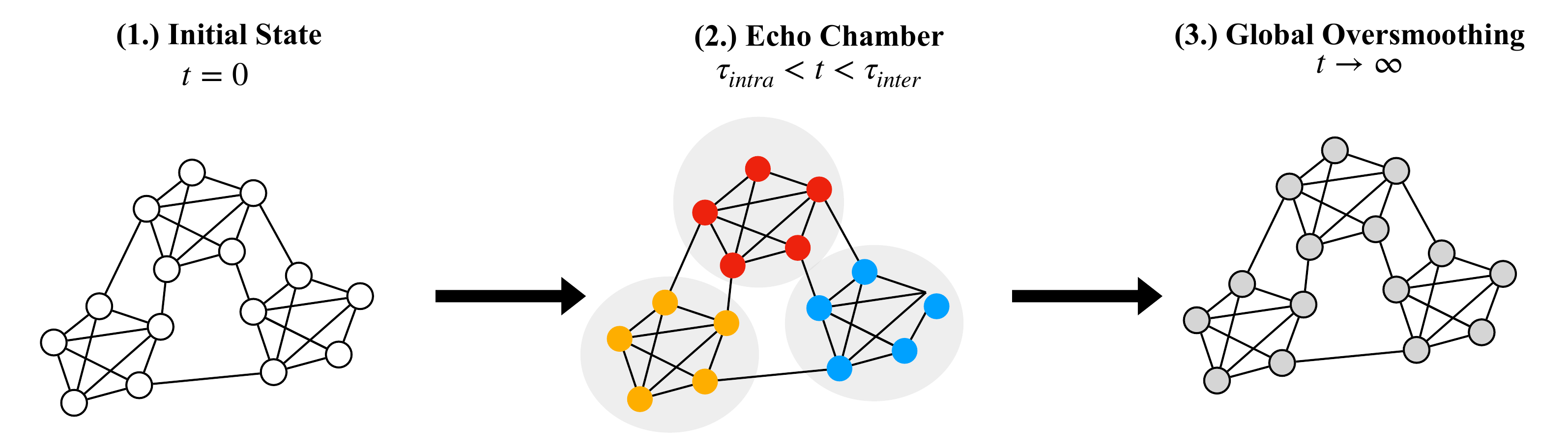}
    \caption{Three regions of GNN message passing:
     initial state~(\textbf{1}), echo chamber where intra-community 
collapse precedes inter-community separation~(\textbf{2}),
    and global oversmoothing~(\textbf{3}).}
    \label{fig:echo_chamber_stages}
\end{figure}

\paragraph{Our contributions are:}
\begin{itemize}
\item We identify and formalise the echo chamber effect, a failure
mode of GNN message passing on modular graphs in which intra-community
representations collapse while inter-community separation persists, which cannot be detected by standard diagnostics.
\item We propose ECI, a diagnostic that stratifies pairwise distances
by community membership, and prove that Dirichlet energy can vanish
while ECI remains bounded away from zero.
\item We characterise the consequences of the echo chamber for node
classification, showing it benefits homophilic settings and harms
heterophilic ones, and make the dichotomy quantitative through a
Fisher ratio and a classifier-margin bound.
\item We propose CASP, a computationally efficient plugin that decouples
intra- and inter-community aggregation and learns their beneficial balance
from label structure. CASP serves as a theory-motivated
intervention for validating the practical relevance of ECI.
\end{itemize}
\section{Preliminaries}

\paragraph{Notations.} Let $\G = (\mathcal{V}, \mathcal{E})$ be an
undirected connected graph with $n = |\mathcal{V}|$ nodes and
$|\mathcal{E}|$ edges. We denote the adjacency matrix by $\bA$ and the
degree matrix by $\bD$. Two propagation operators appear in our analysis.
The symmetric normalised adjacency $\Atilde = \bD^{-1/2}(\bA+\bI)\bD^{-1/2}$
is used in standard GCN-style layers~\cite{kipf2017semi} and in our
empirical experiments. The row-stochastic adjacency $\Ahat = \bD^{-1}\bA$ is used in our
theoretical analysis: as a Markov operator, its mixing dynamics are
analytically tractable~\cite{oono2020graph}. A GCN-style layer takes
the form $\bH^{(k)} = \sigma(\Atilde \bH^{(k-1)}\mathbf{W}^{(k)})$,
which we treat as a representative instance of message passing. Following prior work that simplifies GNN propagation for analytical tractability~\cite{wu2019simplifying,hevapathige2026beyond}, we work with the linearised iteration $\bH^{(k)} = \Ahat \bH^{(k-1)}$ in our structural analysis, isolating propagation dynamics from learning.

For our theoretical analysis we consider graphs with modular structure,
which we model via the stochastic block model
(SBM) \citep{holland1983stochastic}, where intra-block pairs connect with probability
$p_{\intra}$ and cross-block pairs with $p_{\inter} \ll p_{\intra}$
across $m$ balanced communities, denoted $\G \sim \mathrm{SBM}(n,m,p_{\intra},p_{\inter})$.

\paragraph{Oversmoothing and its diagnostics.}
In GNNs, repeated application of $\Ahat$ drives node representations toward
a common stationary vector, a phenomenon known as \emph{oversmoothing}
\citep{rusch2023survey}. A standard way to quantify this is via Dirichlet
energy~\citep{cai2020note}, defined for a feature matrix
$\bH \in \R^{n \times d}$ with $\bH_i \in \R^d$ the representation of node
$i$ as
$E(\bH) = \tfrac{1}{2}\sum_{(i,j)\in\mathcal{E}}
\|\bH_i/\sqrt{1+d_i} - \bH_j/\sqrt{1+d_j}\|^2$,
where $d_i$ is the degree of node $i$. As depth increases, $E(\bH) \to 0$
under oversmoothing. The edge-normalised variant
$
\bar E(\bH)
= \frac{1}{|\mathcal{E}|}
\sum_{(i,j)\in\mathcal{E}}
\left\|\frac{\bH_i}{\sqrt{1+d_i}} - \frac{\bH_j}{\sqrt{1+d_j}}\right\|^2
$
differs only by a constant and allows comparison across graphs of different
sizes. Chen et al.\cite{chen2020measuring} proposed Mean Average Distance (MAD), the
average pairwise cosine distance between node representations, as a
complementary measure. Both metrics aggregate globally, thus representations
may have already collapsed within communities while remaining distinct
across them, a failure neither can detect.

\section{The Echo Chamber Effect}
\label{sec:eci}

In this section, we formalize the echo chamber effect and propose a diagnostic 
to make it measurable. In a modular graph, message passing occurs through two 
diffusion processes that operate at vastly different speeds. Information rapidly 
spreads through dense connections within communities, while it spreads slowly 
across sparse connections between different communities (see
Figure~\ref{fig:echo_chamber_stages}). The MAD metric blends these two processes,
failing to capture the asymmetry between them. When stratified by community
membership, this asymmetry becomes apparent.

\begin{definition}[Community MAD and ECI]
\label{def:mad_eci}
Let $\mathcal{C} = \{C_1,\ldots,C_m\}$ be the community partition, $c(i)$ the
community of node $i$, and $\mathcal{P}_{\intra}$, $\mathcal{P}_{\inter}$ the
edge sets within and across communities. The intra- and inter-community mean
average distances are:
\[
  \MAD_{\intra}^{(k)} = \frac{1}{|\mathcal{P}_{\intra}|}
    \sum_{(i,j)\in\mathcal{P}_{\intra}}\|\bH_i^{(k)}-\bH_j^{(k)}\|_2, \qquad
  \MAD_{\inter}^{(k)} = \frac{1}{|\mathcal{P}_{\inter}|}
    \sum_{(i,j)\in\mathcal{P}_{\inter}}\|\bH_i^{(k)}-\bH_j^{(k)}\|_2,
\]
and the \emph{Echo Chamber Index} (ECI) is:
\[
\ECI^{(k)} =
\frac{\MAD_{\inter}^{(k)}}{\MAD_{\intra}^{(k)}+\nu},
\]
where $\nu>0$ is a small numerical stabiliser.
\end{definition}

Three regions emerge naturally: \textbf{Echo Chamber} ($\ECI \gg 1$), where
intra-community distances have collapsed while inter-community ones remain
large; \textbf{Uniform Mixing} ($\ECI \approx 1$), where both distance scales
decay at comparable rates and no structural asymmetry is present; and
\textbf{Global Oversmoothing} ($\ECI \approx 0$), where all pairwise
distinctions vanish. Dirichlet energy and MAD aggregate globally and therefore
cannot distinguish among these three states.

\begin{definition}[Echo chamber effect and its window]
\label{def:window}
The \emph{echo chamber effect} occurs at depth $k$ when
\[
\frac{\MAD_{\intra}^{(k)}}{\MAD_{\intra}^{(0)}} \ll 
\frac{\MAD_{\inter}^{(k)}}{\MAD_{\inter}^{(0)}}.
\]
The \emph{echo chamber window} $[\tau_{\intra},\tau_{\inter})$ captures the 
depths over which this region persists, where
\[
\tau_{\intra}
=
\min\left\{
k:\MAD_{\intra}^{(k)}
\leq
\gamma_{\intra}\MAD_{\intra}^{(0)}
\right\},
\qquad
\tau_{\inter}
=
\min\left\{
k:\MAD_{\inter}^{(k)}
\leq
\gamma_{\inter}\MAD_{\inter}^{(0)}
\right\},
\]
for thresholds $\gamma_{\intra},\gamma_{\inter}\in(0,1)$.
\end{definition}

We assume $\tau_{\intra}<\tau_{\inter}$, meaning that intra-community
representations collapse before inter-community ones. This occurs when
intra-community edges are denser than inter-community edges and the initial
features are not already more homogeneous across communities than within
them. Pathological violations are theoretically possible.

On a modular graph, the two timescales are determined by both graph structure
and the initial feature distribution. Intra-community edges are dense, so
within-community representations converge quickly, while inter-community edges
are sparse, so cross-community mixing is slow. On
$\G\sim\mathrm{SBM}(n,m,p_{\intra},p_{\inter})$, the contrast
$p_{\intra}\gg p_{\inter}$ means that cross-community information flows more
slowly than within-community information. As this contrast grows, the echo
chamber window widens and can exceed the number of layers used in practice.
This establishes a sufficient regime in which practical GNNs
may operate inside the echo chamber window throughout much of their
aggregation process.

\subsection{The Diagnostic Gap: Why Dirichlet Energy Can Miss Community Separation}
\label{sec:gap}

The above discrepancy is not an artefact but follows from the fact that
Dirichlet energy and ECI average over different objects. Dirichlet energy
measures representation variation along edges, whereas ECI compares intra- and
inter-community distances over node pairs. The following results formalise this
distinction. The proofs are deferred to the Appendix.

\begin{lemma}[Inter-community Energy Bound]
\label{lem:dirichlet_bound}
Let $\bar E_{\inter}(\bH)$ denote the restriction of the normalized Dirichlet
energy to inter-community edges. Then
\[
\bar E_{\inter}(\bH)
\leq
\frac{|\mathcal E_{\inter}|}{|\mathcal E|}
\sup_{(i,j)\in\mathcal E_{\inter}}
\left\|
\frac{\bH_i}{\sqrt{1+d_i}}
-
\frac{\bH_j}{\sqrt{1+d_j}}
\right\|^2.
\]
\end{lemma}

The inter-community contribution to Dirichlet energy scales with the fraction
of inter-community edges, so even large cross-community differences have
negligible effect when such edges are sparse.

\begin{proposition}[Diagnostic gap]
\label{prop:dirichlet}
Let $\G\sim\mathrm{SBM}(n,2,p_{\intra},\epsilon p_{\intra})$ with balanced
communities, where $\epsilon=p_{\inter}/p_{\intra}\in(0,1)$.
Let $c(i)\in\{1,2\}$ denote the community of node $i$ and
$\bar{\bh}_c^{(k)}$ the centroid of community $c$ at depth $k$.
Suppose that for any node $i$ at some depth $k$,
\[
\bH_i^{(k)}
=
\bar{\bh}_{c(i)}^{(k)}
+
\boldsymbol{\xi}_i^{(k)},
\qquad
\|\boldsymbol{\xi}_i^{(k)}\|\leq\eta,
\]
and that
$\|\bar{\bh}_1^{(k)}-\bar{\bh}_2^{(k)}\|\geq\Delta>0$
with $\eta<\Delta/2$. Then, in the limit $n\to\infty$ followed by
$\epsilon\to0$,
\[
\bar E(\bH^{(k)})\to0,
\qquad
\ECI^{(k)}
\geq
\frac{\Delta-2\eta}{2\eta+\nu}
>0.
\]
\end{proposition}

The contribution of inter-community edges to Dirichlet energy scales with their
fraction, which vanishes as $\epsilon\to0$, while inter-community pairwise
distances remain bounded away from zero, yielding a persistent gap captured
by ECI. The condition $\eta<\Delta/2$ is a sufficient
concentration condition for Proposition~\ref{prop:dirichlet}, and it is not
assumed to hold for every community pair or at every depth in a trained
network.

Figure~\ref{fig:layer-metrics} provides empirical confirmation of the 
diagnostic gap established in Proposition~\ref{prop:dirichlet} and 
Lemma~\ref{lem:dirichlet_bound}. Across both the synthetic SBM, with initial 
features drawn independently from $\mathcal{N}(0,I)$, and Coauthor-CS, 
Dirichlet energy and MAD decay towards zero while ECI remains elevated. This
demonstrates that small global Dirichlet energy need not imply vanishing
inter-community separation.

\subsection{Echo Chamber Permanence Under Feature Retention}
\label{sec:permanence}

Standard oversmoothing remedies introduce either a residual connection to the
input (\emph{Initial Feature Retention}, IFR) or a skip to the previous layer
(\emph{Previous Layer Retention}, PLR).
We show that IFR can preserve the echo chamber
asymptotically, whereas PLR can produce arbitrarily long but finite echo
chamber windows on sufficiently modular graphs.

\begin{definition}[IFR / PLR]
\label{def:retention}
A GNN satisfies \textbf{IFR} if
\[
\mathbf{H}^{(k)}
=
f_k(\hat{\mathbf{A}},\mathbf{H}^{(0)})
+
\alpha_k\mathbf{H}^{(0)}
\]
with $\alpha_k\geq\alpha_{\min}>0$, where $\alpha_{\min}$ is a uniform
lower bound on the residual coefficient across all layers, and
\[
\eta_k
=
\sup_{i,j}
\|[f_k]_i-[f_k]_j\|/2
\]
measures the maximum pairwise variation of the propagated component $f_k$ at
depth $k$. It satisfies \textbf{PLR} if
\[
\mathbf{H}^{(k)}
=
(1-\beta_k)\hat{\mathbf{A}}\mathbf{H}^{(k-1)}
+
\beta_k\mathbf{H}^{(k-1)}
\]
with $\beta_k\geq\beta_{\min}>0$, where $\beta_{\min}$ is a uniform lower
bound on the skip coefficient across all layers.
\end{definition}

IFR subsumes architectures with a fixed residual to
$\mathbf{H}^{(0)}$ for which $\inf_k\alpha_k>0$, while PLR subsumes
architectures with skip connections to the preceding layer.

\begin{theorem}[Echo chamber permanence]
\label{thm:permanence}
Let $\mathcal{G}\sim\mathrm{SBM}(n,2,p,\varepsilon p)$ with
initial features satisfying $\|\boldsymbol{\xi}_i^{(0)}\|\leq\eta_0$,
$\|\bar{\mathbf{h}}_1^{(0)}-\bar{\mathbf{h}}_2^{(0)}\|\geq\Delta$,
and $\eta_0<\alpha_{\min}\Delta/[2(1+\alpha_{\min})]$.
\begin{enumerate}[label=(\alph*),leftmargin=*]
\item \textbf{IFR:} for all $k$ with $\eta_k\leq\eta_0$,
\[
    \mathrm{ECI}^{(k)}
    \geq\frac{\alpha_{\min}(\Delta-2\eta_0)-2\eta_0}
             {4\eta_0+\nu}>0.
\]
The echo chamber is \emph{permanent}:
$\liminf_{k\to\infty}\mathrm{ECI}^{(k)}>0$.
\item \textbf{PLR:} the inter-community collapse depth satisfies
$\tau_{\mathrm{inter}}(\beta_{\min},\varepsilon)\to\infty$
as $\varepsilon\to 0$. The echo chamber window widens without
bound; for any fixed $K$, $\mathrm{ECI}^{(k)}>0$ for all
$k\leq K$ when $\varepsilon$ is sufficiently small.
\end{enumerate}
\end{theorem}

The condition $\eta_0<\alpha_{\min}\Delta/[2(1+\alpha_{\min})]$ is sufficient rather than universal. We make the architectural correspondence explicit. APPNP satisfies IFR directly: $\mathbf{H}^{(k+1)}=(1-\alpha)\hat{\mathbf{A}}\mathbf{H}^{(k)}+\alpha\mathbf{H}^{(0)}$ has constant retention weight $\alpha_{\min}=\alpha$. GCNII, $\mathbf{H}^{(k+1)}=\sigma\big(((1-\alpha_k)\hat{\mathbf{A}}\mathbf{H}^{(k)}+\alpha_k\mathbf{H}^{(0)})((1-\beta_k)\mathbf{I}+\beta_k\mathbf{W}^{(k)})\big)$, is covered only under its linearised, identity-weight form: dropping $\sigma$ and setting $\beta_k=\log(\lambda/k+1)\to0$ recovers IFR with weight $\alpha_k$, but this is an approximation to the trained architecture, not an exact match. GPR-GNN reduces to APPNP at $\gamma_k=\alpha(1-\alpha)^k$ and is covered only at that fixed-coefficient setting; free $\gamma_k$ does not guarantee $\alpha_{\min}>0$. JK-Net retains earlier layers via concatenation rather than a per-step residual, so $\|\mathbf{H}^{\mathrm{cat}}_i-\mathbf{H}^{\mathrm{cat}}_j\|^2\geq\|\mathbf{H}^{(m)}_i-\mathbf{H}^{(m)}_j\|^2$ for any retained block $m$ lower-bounds $\mathrm{MAD}_{\mathrm{inter}}$ but not $\mathrm{MAD}_{\mathrm{intra}}$ from above, so it does not directly satisfy PLR. The theorem thus gives mechanistic connections to these architectures rather than exact coverage of their complete trained forms.

\subsection{Task-Conditional Consequences: Node Classification}
\label{sec:tasks}

Operating in the echo chamber reduces intra-community variance and preserves
inter-community separation. We consider the implications of this region for
node classification, where performance depends on how community structure
aligns with class labels. The following theorem formalises this relationship.
The proof is provided in the Appendix.

\begin{theorem}[Task-conditional consequence]
\label{thm:task}
Consider node classification on
$\G\sim\mathrm{SBM}(n,m,p_{\intra},p_{\inter})$.
Suppose
\[
\mathbf{H}_i
=
\bar{\mathbf{h}}_{c(i)}
+
\boldsymbol{\xi}_i,
\]
with
$\|\boldsymbol{\xi}_i\|\leq\eta$,
$\|\bar{\mathbf{h}}_c-\bar{\mathbf{h}}_{c'}\|\geq\Delta>0$
for all $c\neq c'$, and $\eta<\Delta/2$, where $\mathrm{FR}$ denotes the
Fisher Ratio.
\begin{enumerate}[label=(\alph*),topsep=2pt,itemsep=1pt]
\item \textbf{Homophily} ($c(i)=c(j)\iff y_i=y_j$):
\[
\mathrm{FR}(c,c')
\geq
\frac{(\Delta-2\eta)^2}{8\eta^2},
\qquad
\mathrm{FR}(c,c')
:=
\frac{\|\bar{\mathbf{H}}_c-\bar{\mathbf{H}}_{c'}\|^2}
{\mathrm{tr}(\mathrm{Cov}(\mathbf{H}_c))
+
\mathrm{tr}(\mathrm{Cov}(\mathbf{H}_{c'}))},
\]
where
\[
\bar{\mathbf{H}}_c
=
\frac{1}{|C_c|}
\sum_{i\in C_c}\mathbf{H}_i.
\]

\item \textbf{Heterophily}
($\exists\,i,j\in C_s$ with $y_i=a\neq b=y_j$):
\[
\|\mathbf{H}_i-\mathbf{H}_j\|
\leq
2\eta,
\qquad
\|\mathbf{w}_a-\mathbf{w}_b\|
\geq
\frac{\gamma}{\eta},
\]
for any multiclass linear classifier
$f_c(\mathbf{h})=\mathbf{w}_c^\top\mathbf{h}+b_c$
that classifies both $i$ and $j$ correctly with margin $\gamma>0$.
Equivalently, if $\|\mathbf{w}_a-\mathbf{w}_b\|\leq B$, then no margin
$\gamma>B\eta$ is achievable.
\end{enumerate}
\end{theorem}

Under homophily, where communities coincide with classes, as $\eta\to0$ the
intra-class variance vanishes while inter-class separation remains bounded.
The Fisher ratio therefore grows without bound, improving class separability
and confirming that the echo chamber can be beneficial under homophily. In
contrast, under heterophily, nodes with different labels can belong to the same
community and their representations collapse:
$\|\mathbf{H}_i-\mathbf{H}_j\|\leq2\eta$. Maintaining a fixed margin
$\gamma>0$ then requires
$\|\mathbf{w}_a-\mathbf{w}_b\|\geq\gamma/\eta$, demonstrating that the echo
chamber makes margin-based classification increasingly difficult as
$\eta\to0$.
\section{Community-Aware Split Propagation (CASP)}
\label{sec:casp}

Our theoretical analysis shows that standard propagation can enter an echo chamber
region, yet the relative strength of these two effects is fixed by
the graph structure. Theorem~\ref{thm:task} shows this phenomenon is beneficial under 
homophily but harmful under heterophily, and the direction depends on whether 
intra- or inter-community edges are more label-consistent. This motivates 
explicit, data-driven control over the two mixing processes. We do this by 
decomposing the adjacency matrix and weighting the two components separately.

\begin{definition}[CASP]
\label{def:casp}
Let $\mathcal{C}=\{C_1,\ldots,C_m\}$ be a community partition obtained 
by any community detection algorithm.
We define community-aware split propagation as:
\[
\bH^{(k)} =
\alpha_{\intra}\mathbf{P}_{\intra}\bH^{(k-1)}
+
\alpha_{\inter}\mathbf{P}_{\inter}\bH^{(k-1)},
\]
where $\mathbf{P}$ is the propagation operator of the backbone GNN and
$\mathbf{P}_{\intra}$, $\mathbf{P}_{\inter}$ are its restrictions to
intra- and inter-community edges, with
$\mathbf{P}=\mathbf{P}_{\intra}+\mathbf{P}_{\inter}$.
\end{definition}

For GCN, $\mathbf{P}=\tilde{\mathbf{A}}=\mathbf{D}^{-1/2}(\mathbf{A}+\mathbf{I})\mathbf{D}^{-1/2}$ 
and the layer becomes $\bH^{(k)}=(\alpha_{\intra}\tilde{\mathbf{A}}_{\intra}+
\alpha_{\inter}\tilde{\mathbf{A}}_{\inter})\bH^{(k-1)}\mathbf{W}^{(k)}$.  We define the \textbf{propagation ratio} $r=\alpha_{\intra}/\alpha_{\inter}$, 
which controls the balance between the two: increasing $r$ strengthens 
intra-community aggregation, while decreasing $r$ promotes cross-community mixing. The high-level architecture of CASP is depicted in Figure \ref{fig:casp} (Appendix).

\subsection{Mass-Preserving Parameterisation}

We determine $(\alpha_{\intra},\alpha_{\inter})$ by requiring CASP to preserve
expected aggregation mass. Since $p_i=k_i^{\intra}/d_i$ varies across nodes,
a per-node constraint cannot be satisfied with shared scalars. Instead, we
enforce the constraint in expectation: under the balanced SBM, $p_i$
concentrates around its mean $p:=|\mathcal{E}_{\intra}|/|\mathcal{E}|$,
so we require
\[
\alpha_{\intra}p+\alpha_{\inter}(1-p)=1.
\]
Combining this with the propagation ratio
$r$, which controls the balance between
intra- and inter-community aggregation, yields
\begin{equation}
\alpha_{\intra}=\frac{r}{rp+(1-p)},
\qquad
\alpha_{\inter}=\frac{1}{rp+(1-p)}.
\label{eq:alphas}
\end{equation}
Increasing $r$ strengthens intra-community aggregation while decreasing $r$
promotes cross-community mixing.
We establish three key properties of this parameterisation in
Appendix~\ref{app:casp_proofs}: (i) at $r=1$, CASP recovers standard
propagation; (ii) $\alpha_{\intra}$ increases and $\alpha_{\inter}$ decreases
monotonically in $r$, providing stable interpretable control with no abrupt
transitions; and (iii) each node aggregates the same expected total mass as
under standard propagation for any $r>0$, so CASP rebalances the
\emph{source} of information without changing its \emph{scale}.
In particular, $r>1$ biases aggregation toward intra-community edges, which
Theorem~\ref{thm:task} shows is beneficial under homophily, while $r<1$
promotes cross-community mixing, beneficial under heterophily.

\subsection{Data-Driven Calibration}
\label{sec:adaptive}

The appropriate direction of $r$ depends on whether intra-community edges are
more label-consistent than inter-community edges. We estimate this directly
from labelled training edges. Specifically,
\[
h_{\intra}
=
\frac{|\{(i,j)\in\mathcal E_{\intra}^{\mathrm{tr}}: y_i=y_j\}|}
{|\mathcal E_{\intra}^{\mathrm{tr}}|},
\qquad
h_{\inter}
=
\frac{|\{(i,j)\in\mathcal E_{\inter}^{\mathrm{tr}}: y_i=y_j\}|}
{|\mathcal E_{\inter}^{\mathrm{tr}}|},
\]
and $\delta = h_{\intra}-h_{\inter}$. We then set $r=\exp(\theta\delta)$
with learnable $\theta>0$. The exponential ensures $r>0$ for all
$\delta$, maps $\delta=0$ to $r=1$ (recovering standard propagation), and
grows or decays symmetrically around this baseline. A positive $\delta$
indicates that intra-community edges are more label-consistent, in which case
$r>1$ emphasises within-community propagation. Conversely, when $\delta<0$,
cross-community edges are more informative and $r<1$ promotes mixing.

\subsection{Complexity Analysis}
\label{sec:complexity}

CASP introduces one scalar parameter $\theta$. Preprocessing consists of
community detection, edge splitting, and computing the statistics $p$ and
$\delta$. Community detection is performed once as a preprocessing step using
the Louvain algorithm~ with $O(|\mathcal{E}|)$
complexity, followed by $O(|\mathcal{E}|)$ for edge splitting. $\delta$ is
computed directly from labelled training edges in $O(|\mathcal{E}|)$,
introducing no additional learned parameters or modules. During training,
each layer performs two sparse propagation steps over $\Ahat_{\intra}$ and
$\Ahat_{\inter}$. Since
$|\mathcal{E}_{\intra}|+|\mathcal{E}_{\inter}|=|\mathcal{E}|$, the total
message-passing cost remains $O(|\mathcal{E}|d)$ per layer up to a constant
factor, with memory linear in the number of edges.
\section{Experiments}

\subsection{Experimental Setup}

We provide experiments on diverse aspects of our work, focusing on the echo
chamber effect, the ECI metric, and the CASP plugin. We evaluate CASP on eleven
node classification benchmarks: five
homophilic~\cite{bojchevski2018deep} and six
heterophilic~\cite{pei2020geom,platonov2023critical} datasets. We augment five
backbone GNNs with CASP, selected to cover a broad design spectrum:
GCN~\cite{kipf2017semi} as the standard message-passing baseline;
APPNP~\cite{gasteiger2019predict} as a propagation-decoupled method;
MixHop~\cite{abu2019mixhop} as a multi-hop aggregation model;
FAGCN~\cite{bo2021beyond} as a spectral architecture; and
DirGNN~\cite{rossi2024edge} as a direction-aware model. Community detection is
performed once as a preprocessing step using the Louvain algorithm.
For node classification, we use ten independently generated
random 60/20/20 train/validation/test splits following
\cite{chen2025graph,hevapathige2026beyond}, with a distinct seed for each
split, and report the mean and standard deviation across the ten runs.
Dataset statistics, model hyperparameters, and implementation details are
provided in Appendix~\ref{sec:exp_des}.

\subsection{Echo Chamber Across Graph Types}

We validate the echo chamber effect and ECI diagnostic on both synthetic and
real-world graphs.

We first sweep $\epsilon=p_{\mathrm{inter}}/p_{\mathrm{intra}}$ on SBM graphs
to test Proposition~\ref{prop:dirichlet}. As $\epsilon$ decreases, ECI peaks
grow and the echo chamber window of Definition~\ref{def:window} widens, while
at $\epsilon=0.4$ ECI stays near unity and the region disappears, confirming
that the diagnostic gap is a consequence of modularity rather than
architecture. GCN eventually exits the echo chamber at sufficient depth,
exhibiting the non-monotonic arc predicted by our two-timescale analysis.
APPNP, GCNII, and GPR-GNN instead plateau at elevated ECI
and do not return to unity within the evaluated depth range, consistent with
persistent echo-chamber behaviour under feature retention.

\begin{figure}[htbp]
\centering
\includegraphics[width=\linewidth]{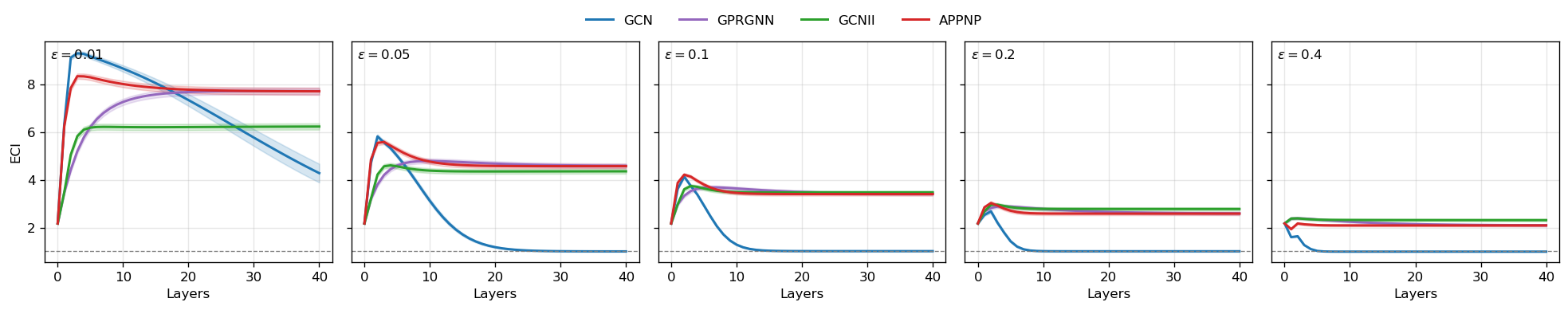}
\caption{SBM sweep ($n=1500$, $m=5$, $\mathbb{E}[\deg]=20$): ECI depth
profiles across $\epsilon=p_{\mathrm{inter}}/p_{\mathrm{intra}}$.}
\label{fig:sbm-sweep}
\end{figure}

Figure~\ref{fig:real-graphs} extends these findings to real-world graphs. On
homophilic graphs, ECI rises and remains elevated throughout, strengthening
with modularity $Q$: Citeseer ($Q=0.89$) and DBLP ($Q=0.76$) sustain the
largest echo chambers. High-modularity heterophilic graphs, Roman-Empire
($Q=0.99$) and Amazon-Ratings ($Q=0.97$), exhibit persistent echo chambers
despite their heterophilic label structure, consistent with
Theorem~\ref{thm:task}, which predicts that this region is harmful in this
setting. At low modularity (Tolokers, $Q=0.52$; Film, $Q=0.50$), ECI stays
near unity, and on Film, GCN drops below unity, suggesting that
inter-community representations become more homogeneous than intra-community
ones. Across all graphs, APPNP, GCNII, and GPR-GNN reach
moderate ECI plateaus and do not exit within the evaluated depth range,
reproducing the finite-depth behaviour observed on the SBM.

\begin{figure}[htbp]
\centering
\includegraphics[width=\linewidth]{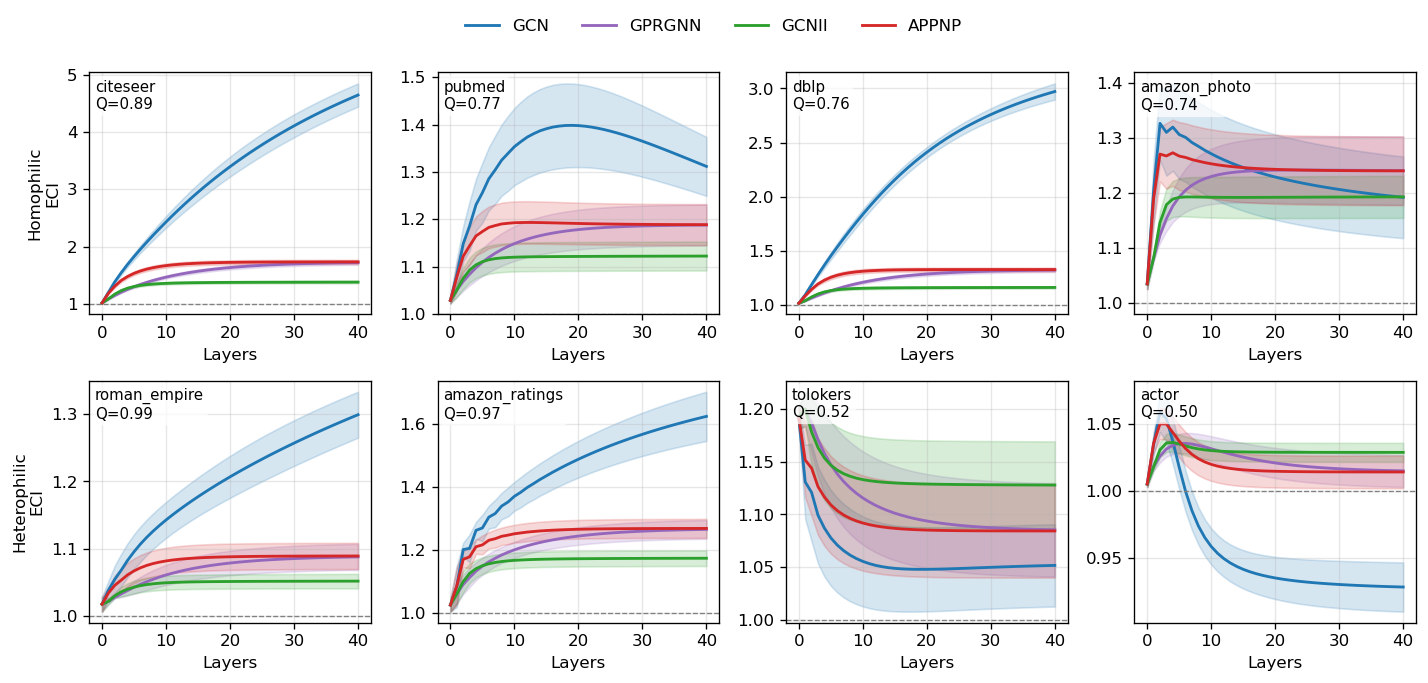}
\caption{ECI profiles for real-world homophilic (top row) and heterophilic
(bottom row) graphs.}
\label{fig:real-graphs}
\end{figure}

\subsection{Node Classification Performance of CASP}

Table~\ref{tab:casp_results} reports node classification performance. CASP improves its corresponding backbone on majority of the
dataset--model combinations. On homophilic datasets, the gains are consistent
with Theorem~\ref{thm:task}(a): stronger intra-community propagation reduces
within-class variance when communities align with labels. Gains are often
larger on heterophilic datasets, where CASP promotes cross-community mixing
through $r<1$. Note that the improvement depends on the backbone:
CASP modifies propagation but does not replace the need for an appropriate
aggregation architecture. The three observed losses are small relative to
the corresponding backbone performance.

\begin{table*}[htbp]
\centering
\caption{Node classification on homophilic and heterophilic benchmarks. Each
base model is paired with its CASP-augmented variant
(\textcolor{darkgreen}{$\uparrow$} gain,
\textcolor{red}{$\downarrow$} loss). ROC-AUC (\%) is reported for
Minesweeper, Tolokers, and Questions; accuracy (\%) is reported elsewhere.
A subset of baseline results is taken from
\cite{chen2025graph,hevapathige2026beyond,mo2025autosgnn}; the remaining
baselines are reproduced using the hyperparameters reported in their original
papers.}
\label{tab:casp_results}
\setlength{\tabcolsep}{4pt}
\renewcommand{\arraystretch}{1.15}
\resizebox{\textwidth}{!}{%
\begin{tabular}{l cccccc ccccc}
\toprule
& \multicolumn{5}{c}{\textbf{Homophilic}}
& \multicolumn{6}{c}{\textbf{Heterophilic}} \\
\cmidrule(lr){2-6} \cmidrule(lr){7-12}
\textbf{Method}
& Cora Full & Citeseer & Pubmed & DBLP & Photo & Film
& Minesweeper & Tolokers & Questions & Amazon-Ratings & Roman-Empire \\
\midrule

GCN
& 68.06\tiny{$\pm$0.98} & 76.68\tiny{$\pm$1.64} & 86.74\tiny{$\pm$0.47} & 83.93\tiny{$\pm$0.84} & 89.30\tiny{$\pm$0.82}
& 30.26\tiny{$\pm$0.79} & 74.79\tiny{$\pm$1.78} & 79.61\tiny{$\pm$0.66} & 65.47\tiny{$\pm$0.88} & 37.99\tiny{$\pm$0.61} & 71.23\tiny{$\pm$0.22} \\
\rowcolor{gray!8}
GCN + CASP
& 70.62\tiny{$\pm$0.56} & 79.36\tiny{$\pm$1.89} & 89.68\tiny{$\pm$0.50} & 87.06\tiny{$\pm$0.87} & 94.83\tiny{$\pm$0.43}
& 37.23\tiny{$\pm$2.17} & 84.39\tiny{$\pm$1.52} & 80.49\tiny{$\pm$0.56} & 76.01\tiny{$\pm$1.29} & 49.57\tiny{$\pm$0.95} & 74.18\tiny{$\pm$0.51} \\
\rowcolor{gray!8}
& \gain{2.56} & \gain{2.68} & \gain{2.94} & \gain{3.13} & \gain{5.53}
& \gain{6.97} & \gain{9.60} & \gain{0.88} & \gain{10.54} & \gain{11.58} & \gain{2.95} \\
\midrule

APPNP
& 70.63\tiny{$\pm$0.51} & 77.89\tiny{$\pm$1.07} & 83.37\tiny{$\pm$0.42} & 85.36\tiny{$\pm$0.40} & 93.74\tiny{$\pm$0.42}
& 36.44\tiny{$\pm$1.78} & 62.62\tiny{$\pm$1.66} & 68.98\tiny{$\pm$0.96} & 65.04\tiny{$\pm$1.42} & 43.48\tiny{$\pm$0.68} & 64.98\tiny{$\pm$0.44} \\
\rowcolor{gray!8}
APPNP + CASP
& 70.97\tiny{$\pm$0.48} & 78.28\tiny{$\pm$1.56} & 88.76\tiny{$\pm$0.35} & 84.62\tiny{$\pm$0.50} & 95.13\tiny{$\pm$0.49}
& 38.15\tiny{$\pm$1.59} & 63.07\tiny{$\pm$1.40} & 72.43\tiny{$\pm$0.89} & 66.20\tiny{$\pm$1.17} & 47.38\tiny{$\pm$1.25} & 68.73\tiny{$\pm$0.34} \\
\rowcolor{gray!8}
& \gain{0.34} & \gain{0.39} & \gain{5.39} & \loss{0.74} & \gain{1.39}
& \gain{1.71} & \gain{0.45} & \gain{3.45} & \gain{1.16} & \gain{3.90} & \gain{3.75} \\
\midrule

MixHop
& 64.85\tiny{$\pm$1.08} & 70.75\tiny{$\pm$2.95} & 80.75\tiny{$\pm$2.29} & 84.27\tiny{$\pm$0.31} & 94.83\tiny{$\pm$0.41}
& 32.22\tiny{$\pm$2.34} & 90.46\tiny{$\pm$0.51} & 74.87\tiny{$\pm$0.61} & 75.90\tiny{$\pm$1.10} & 42.81\tiny{$\pm$0.90} & 81.78\tiny{$\pm$0.52} \\
\rowcolor{gray!8}
MixHop + CASP
& 65.27\tiny{$\pm$1.35} & 74.83\tiny{$\pm$3.09} & 91.13\tiny{$\pm$0.43} & 89.04\tiny{$\pm$1.95} & 94.85\tiny{$\pm$0.54}
& 39.42\tiny{$\pm$1.90} & 90.60\tiny{$\pm$0.48} & 75.29\tiny{$\pm$0.68} & 76.30\tiny{$\pm$0.99} & 44.99\tiny{$\pm$2.47} & 81.88\tiny{$\pm$0.45} \\
\rowcolor{gray!8}
& \gain{0.42} & \gain{4.08} & \gain{10.38} & \gain{4.77} & \gain{0.02}
& \gain{7.20} & \gain{0.14} & \gain{0.42} & \gain{0.40} & \gain{2.18} & \gain{0.10} \\
\midrule

FAGCN
& 70.71\tiny{$\pm$0.74} & 83.51\tiny{$\pm$0.43} & 86.08\tiny{$\pm$0.33} & 85.37\tiny{$\pm$0.89} & 93.67\tiny{$\pm$0.50}
& 31.59\tiny{$\pm$1.37} & 85.19\tiny{$\pm$1.27} & 79.56\tiny{$\pm$0.77} & 72.22\tiny{$\pm$1.76} & 42.63\tiny{$\pm$1.20} & 68.59\tiny{$\pm$1.41} \\
\rowcolor{gray!8}
FAGCN + CASP
& 72.32\tiny{$\pm$0.76} & 81.47\tiny{$\pm$1.57} & 89.77\tiny{$\pm$0.20} & 85.80\tiny{$\pm$1.41} & 95.09\tiny{$\pm$0.44}
& 38.18\tiny{$\pm$1.10} & 85.64\tiny{$\pm$0.64} & 80.18\tiny{$\pm$1.06} & 73.90\tiny{$\pm$2.02} & 45.37\tiny{$\pm$1.56} & 71.43\tiny{$\pm$0.80} \\
\rowcolor{gray!8}
& \gain{1.61} & \loss{2.04} & \gain{3.69} & \gain{0.43} & \gain{1.42}
& \gain{6.59} & \gain{0.45} & \gain{0.62} & \gain{1.68} & \gain{2.74} & \gain{2.84} \\
\midrule

DirGNN
& 67.80\tiny{$\pm$0.53} & 77.71\tiny{$\pm$0.78} & 86.94\tiny{$\pm$0.55} & 81.22\tiny{$\pm$0.54} & 95.38\tiny{$\pm$0.32}
& 35.76\tiny{$\pm$1.68} & 81.52\tiny{$\pm$0.41} & 82.64\tiny{$\pm$0.75} & 59.95\tiny{$\pm$0.79} & 46.66\tiny{$\pm$0.61} & 80.08\tiny{$\pm$1.07} \\
\rowcolor{gray!8}
DirGNN + CASP
& 68.88\tiny{$\pm$0.82} & 78.32\tiny{$\pm$2.89} & 90.03\tiny{$\pm$0.53} & 83.93\tiny{$\pm$1.82} & 95.15\tiny{$\pm$0.45}
& 40.49\tiny{$\pm$1.80} & 89.77\tiny{$\pm$0.56} & 84.06\tiny{$\pm$0.63} & 75.26\tiny{$\pm$1.06} & 49.09\tiny{$\pm$1.19} & 80.12\tiny{$\pm$0.70} \\
\rowcolor{gray!8}
& \gain{1.08} & \gain{0.61} & \gain{3.09} & \gain{2.71} & \loss{0.23}
& \gain{4.73} & \gain{8.25} & \gain{1.42} & \gain{15.31} & \gain{2.43} & \gain{0.04} \\
\bottomrule
\end{tabular}%
}
\end{table*}

\paragraph{Additional Experiments}

Appendix~\ref{app:additional} provides additional experiments related to ECI,
including robustness to the community detection algorithm, partition noise,
initial feature conditions, comparison with existing metrics, and the empirical validation of theoretical assumptions. It also
provides CASP component ablations, robustness to sparse data splits and noisy
partitions, depth analysis, comparisons with existing baselines, empirical
validation of echo-chamber window modulation, scalability analysis, and
sensitivity to the Louvain resolution parameter.
\section{Related Work}

\paragraph{Oversmoothing diagnostics.}
Oversmoothing in GNNs has been extensively studied since Li et
al.~\cite{li2018deeper}. Dirichlet
energy~\cite{cai2020note,bison2025analysis} and
MAD~\cite{chen2020measuring} are the dominant diagnostics, but both reduce
representations to a scalar, treating all graph regions as homogeneous and
failing to distinguish intra- from inter-community convergence. Zhou et
al.~\cite{zhou2020towards} propose the Group Distance Ratio (GDR), which
stratifies distances using class-label-defined groups but requires supervision
and primarily serves as a training objective rather than a structural
diagnostic. Keriven et al.~\cite{keriven2022not} shows that finite
smoothing can contract within-community variation faster than
between-community separation. ECI diagnoses this intermediate regime without
labels. ECI identifies three regions on modular graphs corresponding to the
two-timescale propagation process. Although echo chambers and related metrics
have been studied in social dynamics to describe ideological homogenisation
in online communities~\cite{cinelli2021echo,alatawi2023quantifying,putri2024echo},
our use of the term refers specifically to asymmetric representation collapse
during GNN message passing.

Several methods mitigate oversmoothing by changing propagation
depth or retaining earlier representations. DropEdge~\cite{rong2019dropedge}
randomly removes edges, SkipNode~\cite{lu2024skipnode} allows selected nodes to
bypass convolution, and Shadow-GNN~\cite{zeng2021decoupling} restricts the
receptive-field scope independently of network depth. These approaches can
slow smoothing or preserve earlier information, but they do not explicitly
measure or control the relative rates of intra- and inter-community
convergence.

\paragraph{Community structure and GNN design.}
Several works incorporate community structure into GNN design. Skryagin et
al.~\cite{skryagin2024graph} dynamically cluster nodes at each layer and
normalize within clusters to mitigate oversmoothing. Begga et
al.~\cite{begga2024community} use precomputed partitions to guide aggregation
across homophilic and heterophilic settings. Cluster-GCN
\cite{chiang2019cluster} uses graph partitions for scalable mini-batch
training, thereby favouring intra-cluster propagation within each batch. Unlike CASP, its purpose is computational scalability rather
than explicit control of the intra/inter-community propagation balance.
GDC~\cite{gasteiger2019diffusion} preprocesses the graph using diffusion-based
rewiring, whereas CASP retains the edge set and reweights structurally defined
edge classes during propagation.

Palowitch et al.~\cite{palowitch2022graphworld} and Platonov et
al.~\cite{platonov2023critical} show that label--community alignment predicts
GNN performance more reliably than homophily alone.
DiffPool~\cite{ying2018hierarchical} and
MinCutPool~\cite{bianchi2020spectral} learn soft community assignments for
graph-level tasks. CASP builds on this broader use of
community-aware aggregation but uses a label-calibrated parameter to
control the relative contributions of intra- and inter-community messages.
Its principal role in this work is therefore a theory-motivated intervention
for validating and modulating the propagation behaviour identified by ECI,
rather than a replacement for the underlying backbone architecture.
\section{Conclusion, Limitations, and Future Work}

This paper formalizes the echo chamber effect in GNNs, where
intra-community representations collapse while inter-community separation
persists. We introduced ECI to detect this effect and showed that feature-retention mechanisms can preserve it under the
conditions of our theoretical analysis. Our analysis demonstrated that the
echo chamber can benefit node classification when communities align with
classes but harm it when communities contain different classes. We also
proposed CASP, a lightweight plugin that decouples intra- and inter-community
propagation and learns their balance from label structure. Experiments showed
that CASP improves diverse backbones relative to their
corresponding unmodified versions across most evaluated settings.

Our theoretical analysis assumes a balanced SBM,
centroid-concentrated representations, and linearised
propagation. It therefore does not directly cover the learned weights,
nonlinearities, and training dynamics of complete GNN architectures.
CASP relies on a fixed preprocessing partition whose quality can affect
performance, although our experiments show robustness to substantial
partition noise. Its data-driven calibration also requires labelled training
edges, which may be scarce in low-label settings.
Moreover, CASP modifies propagation but does not replace the
representational capabilities of its backbone, so its absolute performance
remains backbone-dependent. Extending the theory to more general graph
models and trained nonlinear GNNs, jointly learning community structure, and
using ECI to guide GNN architecture design are natural directions for future
work.

\bibliographystyle{unsrt}
{
  \small
  \bibliography{references}
}

\clearpage
\newpage
\setcounter{lemma}{0}
\setcounter{proposition}{0}
\setcounter{theorem}{0}
\setcounter{section}{0}
\renewcommand{\thesection}{\Alph{section}}
\section*{Appendix}
\addcontentsline{toc}{section}{Appendix}

\vspace{0.5em}
\noindent\textbf{Table of Contents}
\vspace{0.3em}

\noindent
\hyperref[app:casp_arch]{A\quad CASP Architecture Overview} \dotfill \pageref{app:casp_arch}\\
\hyperref[app:proofs]{B\quad Proofs of Main Theoretical Results} \dotfill \pageref{app:proofs}\\
\hspace*{1.5em}\hyperref[app:proof_lemma1]{B.1\quad Proof of Lemma~1} \dotfill \pageref{app:proof_lemma1}\\
\hspace*{1.5em}\hyperref[app:proof_prop1]{B.2\quad Proof of Proposition~1} \dotfill \pageref{app:proof_prop1}\\
\hspace*{1.5em}\hyperref[app:proof_thm1]{B.3\quad Proof of Theorem~1} \dotfill \pageref{app:proof_thm1}\\
\hspace*{1.5em}\hyperref[app:proof_thm2]{B.4\quad Proof of Theorem~2} \dotfill \pageref{app:proof_thm2}\\
\hyperref[app:casp_proofs]{C\quad Proofs of CASP Theoretical Properties} \dotfill \pageref{app:casp_proofs}\\
\hyperref[sec:exp_des]{D\quad Experimental Design} \dotfill \pageref{sec:exp_des}\\
\hspace*{1.5em}\hyperref[app:datasets]{D.1\quad Dataset Statistics} \dotfill \pageref{app:datasets}\\
\hspace*{1.5em}\hyperref[app:hyperparams]{D.2\quad Model Hyperparameters} \dotfill \pageref{app:hyperparams}\\
\hspace*{1.5em}\hyperref[app:impl]{D.3\quad Implementation Details} \dotfill \pageref{app:impl}\\
\hyperref[app:additional]{E\quad Additional Experiments} \dotfill \pageref{app:additional}\\
\hspace*{1.5em}\hyperref[app:eci_exp]{E.1\quad Ablation Studies on ECI} \dotfill \pageref{app:eci_exp}\\
\hspace*{3.0em}\hyperref[app:thm_validation]{E.1.1\quad Empirical Validation of Theorem~2} \dotfill \pageref{app:thm_validation}\\
\hspace*{3.0em}\hyperref[app:comm_detection]{E.1.2\quad Robustness to Community Detection} \dotfill \pageref{app:comm_detection}\\
\hspace*{3.0em}\hyperref[app:feature_sensitivity]{E.1.3\quad Effect of Initial Node Features on ECI} \dotfill \pageref{app:feature_sensitivity}\\
\hspace*{3.0em}\hyperref[app:diag_gap]{E.1.4\quad Diagnostic Gap: ECI vs.\ Existing Metrics} \dotfill \pageref{app:diag_gap}\\
\hspace*{3.0em}\hyperref[app:diag_gap]{E.1.5\quad Empirical Validation of Theoretical Assumptions} \dotfill \pageref{app:assumption-validation}\\
\hspace*{1.5em}\hyperref[app:casp_ablation]{E.2\quad Ablation Studies on CASP} \dotfill \pageref{app:casp_ablation}\\
\hspace*
{3.0em}\hyperref[app:casp_components]{E.2.1\quad Component Analysis of CASP} \dotfill \pageref{app:casp_components}\\
\hspace*{3.0em}\hyperref[app:split_robustness]{E.2.2\quad Robustness to Data Splits and Community Detection} \dotfill \pageref{app:split_robustness}\\
\hspace*{3.0em}\hyperref[app:noisy_partitions]{E.2.3\quad Robustness to Noisy Community Partitions} \dotfill \pageref{app:noisy_partitions}\\
\hspace*{3.0em}\hyperref[app:depth_analysis]{E.2.4\quad Depth Analysis and Oversmoothing Robustness} \dotfill \pageref{app:depth_analysis}\\
\hspace*{3.0em}\hyperref[app:baselines]{E.2.5\quad Node Classification Performance with Existing Baselines} \dotfill \pageref{app:baselines}\\
\hspace*{3.0em}\hyperref[app:eci_window]{E.2.6\quad Impact of CASP on ECI and Echo Chamber Window} \dotfill \pageref{app:eci_window}\\
\hspace*{3.0em}\hyperref[sec:scalability]{E.2.7\quad Scalability Analysis} \dotfill \pageref{sec:scalability}\\
\hspace*{3.0em}\hyperref[sec:resolution_ablation]{E.2.8\quad Sensitivity to Resolution Parameter in Louvain Algorithm} \dotfill \pageref{sec:resolution_ablation}\\

\vspace{1em}

\section{CASP Architecture Overview} 
\label{app:casp_arch}

Figure~\ref{fig:casp} illustrates the high-level architecture of CASP.
Given an input graph, CASP first identifies community structure via the Louvain
algorithm as a one-time preprocessing step.
Then, message passing procedure is split into two streams: one
operating along intra-community edges and one along inter-community edges.
The two streams are scaled by weights calibrated directly from the
graph's label structure, allowing CASP to amplify intra-community aggregation
under homophily and promote inter-community mixing under heterophily.

\begin{figure}[htbp]
    \centering
    \includegraphics[width=\linewidth]{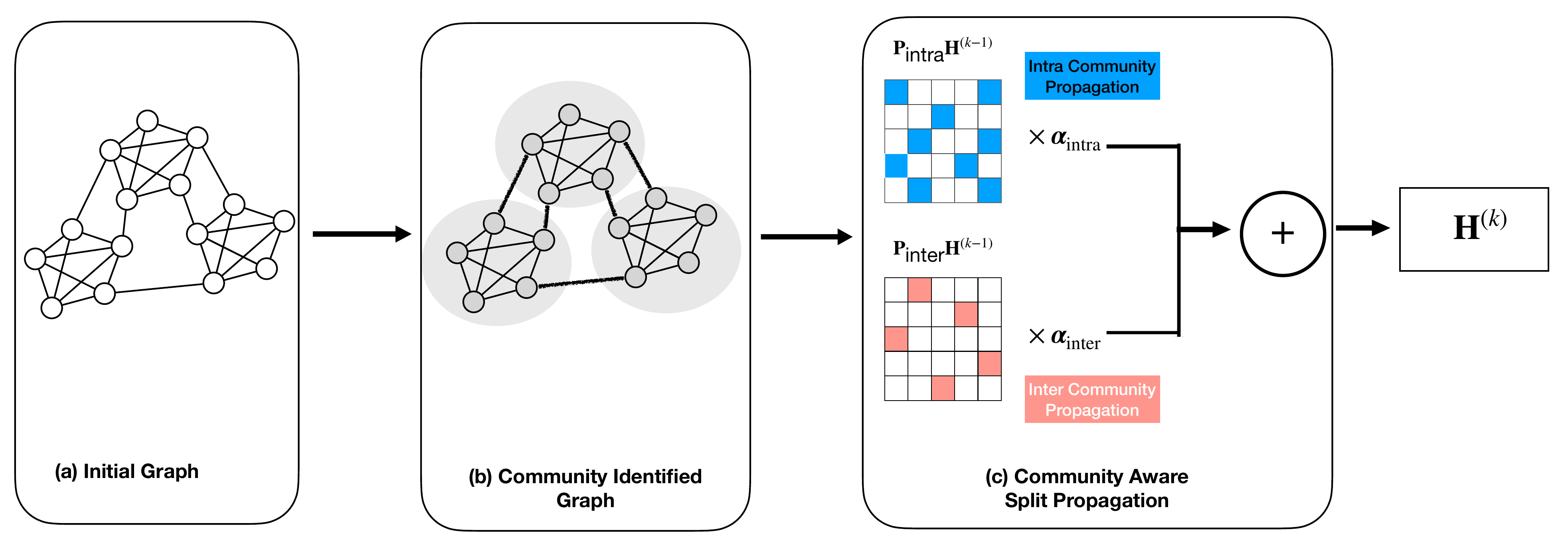}
    \caption{High-level Architecture of CASP
    }
    \label{fig:casp}
\end{figure}

\section{Proofs of Main Theoretical Results}
\label{app:proofs}

In this section we provide proofs for the theoretical results stated in 
the main text, in order of appearance.

\subsection{Proof of Lemma~\ref{lem:dirichlet_bound}}
\label{app:proof_lemma1}

\begin{lemma}[Inter-community Energy Bound]
\label{lem:dirichlet_bound}
Let $\bar E_{\inter}(\bH)$ denote the restriction of the normalized Dirichlet
energy to inter-community edges. Then
\[
\bar E_{\inter}(\bH)
\;\le\;
\frac{|\mathcal E_{\inter}|}{|\mathcal E|}
\sup_{(i,j)\in\mathcal E_{\inter}}
\left\|
\frac{\bH_i}{\sqrt{1+d_i}}
-
\frac{\bH_j}{\sqrt{1+d_j}}
\right\|^2.
\]
\end{lemma}

\begin{proof}
We start by defining:
\[
f_{ij}
=
\left\|
\frac{\bH_i}{\sqrt{1+d_i}}
-
\frac{\bH_j}{\sqrt{1+d_j}}
\right\|^2
\ge 0,
\qquad
M
=
\sup_{(i,j)\in\mathcal E_{\inter}} f_{ij}.
\]
Then
\[
\bar E_{\inter}(\bH)
=
\frac{1}{|\mathcal E|}
\sum_{(i,j)\in\mathcal E_{\inter}} f_{ij}
\le
\frac{1}{|\mathcal E|}
\sum_{(i,j)\in\mathcal E_{\inter}} M
=
\frac{|\mathcal E_{\inter}|}{|\mathcal E|} M,
\]
which proves the claim.
\end{proof}

\subsection{Proof of Proposition~\ref{prop:dirichlet}}
\label{app:proof_prop1}

\begin{proposition}[Diagnostic gap]
\label{prop:dirichlet}
Let $\G\sim\mathrm{SBM}(n,2,p_{\intra},\epsilon p_{\intra})$ with balanced
communities, where $\epsilon = p_{\inter}/p_{\intra} \in (0,1)$.
$c(i)\in\{1,2\}$ denotes the community of node $i$ and
$\bar{\bh}_c^{(k)}$ the centroid of community $c$ at depth $k$.
Suppose that at some depth $k$,
\[
\bH_i^{(k)}=\bar{\bh}_{c(i)}^{(k)}+\boldsymbol{\xi}_i^{(k)},
\qquad
\|\boldsymbol{\xi}_i^{(k)}\|\le \eta,
\]
and that $\|\bar{\bh}_1^{(k)}-\bar{\bh}_2^{(k)}\|\ge \Delta>0$ with $\eta<\Delta/2$.
Then, in the limit $n\to\infty$ followed by $\epsilon\to 0$,
\[
\bar E(\bH^{(k)}) \to 0,
\qquad
\ECI^{(k)} \;\ge\; \frac{\Delta-2\eta}{2\eta+\nu} > 0.
\]
\end{proposition}

\begin{proof}
We can write:
\[
\bar E(\bH^{(k)}) = \bar E_{\intra}(\bH^{(k)}) + \bar E_{\inter}(\bH^{(k)}).
\]

\textit{Step 1: the intra-community contribution vanishes as $n\to\infty$ 
for fixed $\epsilon$.}

Fix $(i,j)\in\mathcal E_{\intra}$. Since $c(i)=c(j)$, we have
$\bar{\bh}_{c(i)}^{(k)}=\bar{\bh}_{c(j)}^{(k)}$, and therefore
\[
\frac{\bH_i^{(k)}}{\sqrt{1+d_i}}-\frac{\bH_j^{(k)}}{\sqrt{1+d_j}}
=
\bar{\bh}_{c(i)}^{(k)}
\left(
\frac{1}{\sqrt{1+d_i}}-\frac{1}{\sqrt{1+d_j}}
\right)
+
\frac{\boldsymbol{\xi}_i^{(k)}}{\sqrt{1+d_i}}
-
\frac{\boldsymbol{\xi}_j^{(k)}}{\sqrt{1+d_j}}.
\]
Hence
\[
\left\|
\frac{\bH_i^{(k)}}{\sqrt{1+d_i}}-\frac{\bH_j^{(k)}}{\sqrt{1+d_j}}
\right\|
\le
M
\left|
\frac{1}{\sqrt{1+d_i}}-\frac{1}{\sqrt{1+d_j}}
\right|
+
\frac{\eta}{\sqrt{1+d_i}}
+
\frac{\eta}{\sqrt{1+d_j}}.
\]
Under SBM degree concentration, for fixed $\epsilon$ we have
$d_{\min}=\Omega(n p_{\intra})$ with high probability by a Chernoff 
bound~\cite{chernoff1952measure}, so
\[
\frac{1}{\sqrt{1+d_i}},\frac{1}{\sqrt{1+d_j}} = O(d_{\min}^{-1/2}),
\qquad
\left|
\frac{1}{\sqrt{1+d_i}}-\frac{1}{\sqrt{1+d_j}}
\right| = O(d_{\min}^{-1}).
\]
Therefore each intra-community edge contributes
$O(d_{\min}^{-1/2})^2 = O(d_{\min}^{-1})$
to the squared norm, and thus
\[
\bar E_{\intra}(\bH^{(k)})
=
\frac{1}{|\mathcal E|}
\sum_{(i,j)\in\mathcal E_{\intra}}
O(d_{\min}^{-1})
=
O(d_{\min}^{-1})
\to 0
\]
as $n\to\infty$.

\textit{Step 2: the inter-community contribution vanishes as $\epsilon\to 0$.}

By Lemma~\ref{lem:dirichlet_bound},
\[
\bar E_{\inter}(\bH^{(k)})
\le
\frac{|\mathcal E_{\inter}|}{|\mathcal E|} C_0.
\]
For the balanced two-block SBM,
\[
\frac{|\mathcal E_{\inter}|}{|\mathcal E|}
\to
\frac{\epsilon}{1+\epsilon}
\]
in expectation, and the same scaling holds with high probability by 
concentration. Hence
\[
\bar E_{\inter}(\bH^{(k)}) = O\!\left(\frac{\epsilon}{1+\epsilon}\right)\to 0
\]
as $\epsilon\to 0$. Combining the two steps yields
$\bar E(\bH^{(k)})\to 0$
in the sequential limit $n\to\infty$ followed by $\epsilon\to 0$.

\textit{Step 3: ECI stays bounded away from zero.}

For any intra-community pair $(i,j)$, since $c(i)=c(j)$,
$\bH_i^{(k)}-\bH_j^{(k)} = \boldsymbol{\xi}_i^{(k)}-\boldsymbol{\xi}_j^{(k)}$,
so $\|\bH_i^{(k)}-\bH_j^{(k)}\| \le 2\eta$, giving
$\MAD_{\intra}^{(k)} \le 2\eta$.
For any inter-community pair $(i,j)$, by the triangle inequality,
$\|\bH_i^{(k)}-\bH_j^{(k)}\| \ge \Delta-2\eta$,
giving $\MAD_{\inter}^{(k)} \ge \Delta-2\eta$.
Therefore
\[
\ECI^{(k)}
=
\frac{\MAD_{\inter}^{(k)}}{\MAD_{\intra}^{(k)}+\nu}
\ge
\frac{\Delta-2\eta}{2\eta+\nu} > 0,
\]
since $\eta<\Delta/2$.
\end{proof}

\subsection{Proof of Theorem~\ref{thm:permanence}}
\label{app:proof_thm1}

\begin{theorem}[Echo chamber permanence]
\label{thm:permanence}
Let $\mathcal{G}\sim\mathrm{SBM}(n,2,p,\varepsilon p)$ with
initial features satisfying $\|\boldsymbol{\xi}_i^{(0)}\|\leq\eta_0$,
$\|\bar{\mathbf{h}}_1^{(0)}-\bar{\mathbf{h}}_2^{(0)}\|\geq\Delta$,
and $\eta_0<\alpha_{\min}\Delta/[2(1+\alpha_{\min})]$.
\begin{enumerate}[label=(\alph*),leftmargin=*]
\item \textbf{IFR:} for all $k$ with $\eta_k\leq\eta_0$,
\[
    \mathrm{ECI}^{(k)}
    \geq\frac{\alpha_{\min}(\Delta-2\eta_0)-2\eta_0}
             {4\eta_0+\nu}>0.
\]
The echo chamber is \emph{permanent}:
$\liminf_{k\to\infty}\mathrm{ECI}^{(k)}>0$.
\item \textbf{PLR:} the inter-community collapse depth satisfies
$\tau_{\mathrm{inter}}(\beta_{\min},\varepsilon)\to\infty$
as $\varepsilon\to 0$. The echo chamber window widens without
bound; for any fixed $K$, $\mathrm{ECI}^{(k)}>0$ for all
$k\leq K$ when $\varepsilon$ is sufficiently small.
\end{enumerate}
\end{theorem}

\begin{proof}
\textbf{(a) IFR.}
Fix $k$ with $\eta_k\leq\eta_0$. For any pair $(i,j)$:
\begin{equation}
    \mathbf{H}_i^{(k)}-\mathbf{H}_j^{(k)}
    =\alpha_k\bigl(\mathbf{H}_i^{(0)}-\mathbf{H}_j^{(0)}\bigr)
    +\bigl([f_k]_i-[f_k]_j\bigr).
    \label{eq:app-diff}
\end{equation}

\emph{Upper bound on $\mathrm{MAD}_{\mathrm{intra}}^{(k)}$.}
For $c(i)=c(j)$: $\|\mathbf{H}_i^{(0)}-\mathbf{H}_j^{(0)}\|
\leq 2\eta_0$. By~\eqref{eq:app-diff}, $\alpha_k\leq 1$,
and the uniform smoothness condition:
\[
    \|\mathbf{H}_i^{(k)}-\mathbf{H}_j^{(k)}\|
    \leq 2\eta_0+2\eta_k \leq 4\eta_0,
\]
so $\mathrm{MAD}_{\mathrm{intra}}^{(k)}\leq 4\eta_0$.

\emph{Lower bound on $\mathrm{MAD}_{\mathrm{inter}}^{(k)}$.}
For $c(i)\neq c(j)$: $\|\mathbf{H}_i^{(0)}-\mathbf{H}_j^{(0)}\|
\geq\Delta-2\eta_0$. By~\eqref{eq:app-diff} and the uniform
smoothness condition applied to all pairs:
\[
    \|\mathbf{H}_i^{(k)}-\mathbf{H}_j^{(k)}\|
    \geq\alpha_{\min}(\Delta-2\eta_0)-2\eta_0,
\]
which is positive by the condition
$\eta_0<\alpha_{\min}\Delta/[2(1+\alpha_{\min})]$.
Hence $\mathrm{MAD}_{\mathrm{inter}}^{(k)}\geq
\alpha_{\min}(\Delta-2\eta_0)-2\eta_0>0$.

\emph{ECI bound.} Combining:
\[
    \mathrm{ECI}^{(k)}
    =\frac{\mathrm{MAD}_{\mathrm{inter}}^{(k)}}
          {\mathrm{MAD}_{\mathrm{intra}}^{(k)}+\nu}
    \geq\frac{\alpha_{\min}(\Delta-2\eta_0)-2\eta_0}
             {4\eta_0+\nu}>0.
\]
The bound is independent of $k$, so
$\liminf_{k\to\infty}\mathrm{ECI}^{(k)}>0$.

\medskip
\textbf{(b) PLR.}
The PLR update applies
$\mathbf{M}=(1-\beta_{\min})\hat{\mathbf{A}}+\beta_{\min}
\mathbf{I}$ per layer. In the $n\to\infty$ limit,
$\hat{\mathbf{A}}$ has community-separation eigenvalue
$\lambda_2=1-2\varepsilon$~\cite{oono2020graph}, so
$\mathbf{M}$ has eigenvalue
\[
    \mu_2=(1-\beta_{\min})\lambda_2+\beta_{\min}
    =1-(1-\beta_{\min})\cdot2\varepsilon>\lambda_2.
\]
Under linearised propagation with initial features aligned with
the community eigenvector, the inter-community centroid
separation satisfies
$\|\bar{\mathbf{h}}_1^{(k)}-\bar{\mathbf{h}}_2^{(k)}\|
\geq\mu_2^k\Delta$.
The inter-community collapse depth is
\[
    \tau_{\mathrm{inter}}
    =\!\left\lceil
        \frac{\log(1/\gamma_{\mathrm{inter}})}{\log(1/\mu_2)}
    \right\rceil\!.
\]
As $\varepsilon\to 0$,
$\log(1/\mu_2)=(1-\beta_{\min})\cdot2\varepsilon
+O(\varepsilon^2)\to 0$,
so $\tau_{\mathrm{inter}}\to\infty$.
The intra-community collapse depth $\tau_{\mathrm{intra}}$ is
governed by the intra-community spectral gap
$\Theta(p_{\mathrm{intra}})$, independent of $\varepsilon$,
so the echo chamber window
$[\tau_{\mathrm{intra}},\tau_{\mathrm{inter}})$ widens without
bound and $\mathrm{ECI}^{(k)}>0$ throughout.
\end{proof}

\subsection{Proof of Theorem~\ref{thm:task}}
\label{app:proof_thm2}

\begin{theorem}[Task-conditional consequence]
\label{thm:task}
Consider node classification on $\G\sim\mathrm{SBM}(n,m,p_{\intra},p_{\inter})$.
Suppose $\mathbf{H}_i = \bar{\mathbf{h}}_{c(i)} + \boldsymbol{\xi}_i$ with
$\|\boldsymbol{\xi}_i\| \le \eta$, $\|\bar{\mathbf{h}}_c - \bar{\mathbf{h}}_{c'}\|
\ge \Delta > 0$ for all $c \neq c'$, and $\eta < \Delta/2$,
where $\mathrm{FR}$ denotes the Fisher Ratio.
\begin{enumerate}[label=(\alph*),topsep=2pt,itemsep=1pt]
\item \textbf{Homophily} ($c(i)=c(j) \iff y_i = y_j$):
\[
\mathrm{FR}(c,c') \ge \frac{(\Delta - 2\eta)^2}{8\eta^2}, \qquad
\mathrm{FR}(c,c') := \frac{\|\bar{\mathbf{H}}_c - \bar{\mathbf{H}}_{c'}\|^2}
{\mathrm{tr}(\mathrm{Cov}(\mathbf{H}_c)) + \mathrm{tr}(\mathrm{Cov}(\mathbf{H}_{c'}))},
\]
where $\bar{\mathbf{H}}_c = \frac{1}{|C_c|}\sum_{i \in C_c} \mathbf{H}_i$.
\item \textbf{Heterophily} ($\exists\, i,j \in C_s$ with $y_i = a \neq b = y_j$):
\[
\|\mathbf{H}_i - \mathbf{H}_j\| \le 2\eta, \qquad
\|\mathbf{w}_a - \mathbf{w}_b\| \ge \frac{\gamma}{\eta},
\]
for any multiclass linear classifier $f_c(\mathbf{h}) = \mathbf{w}_c^\top \mathbf{h} + b_c$
that classifies both $i$ and $j$ correctly with margin $\gamma > 0$.
Equivalently, if $\|\mathbf{w}_a - \mathbf{w}_b\| \le B$, then no margin $\gamma > B\eta$
is achievable.
\end{enumerate}
\end{theorem}

\begin{proof}
Both parts follow from $\mathbf{H}_i = \bar{\mathbf{h}}_{c(i)} + \boldsymbol{\xi}_i$ 
with $\|\boldsymbol{\xi}_i\| \le \eta$.

\textit{(a).}
Under $c(i)=c(j) \iff y_i = y_j$, communities coincide with classes, so
$\bar{\mathbf{H}}_c$ and $\mathrm{Cov}(\mathbf{H}_c)$ are computed over $C_c$.
Since $\mathbf{H}_i = \bar{\mathbf{h}}_c + \boldsymbol{\xi}_i$,
\[
\mathbf{H}_i - \bar{\mathbf{H}}_c = \boldsymbol{\xi}_i - \bar{\boldsymbol{\xi}}_c,
\]
giving $\|\mathbf{H}_i - \bar{\mathbf{H}}_c\| \le 2\eta$ and hence
$\mathrm{tr}(\mathrm{Cov}(\mathbf{H}_c)) \le 4\eta^2$. Applying the same bound to $c'$
and summing yields
\[
\mathrm{tr}(\mathrm{Cov}(\mathbf{H}_c)) + \mathrm{tr}(\mathrm{Cov}(\mathbf{H}_{c'})) \le 8\eta^2.
\]
Moreover,
\[
\bar{\mathbf{H}}_c - \bar{\mathbf{H}}_{c'} = (\bar{\mathbf{h}}_c - \bar{\mathbf{h}}_{c'}) 
+ (\bar{\boldsymbol{\xi}}_c - \bar{\boldsymbol{\xi}}_{c'}),
\]
so $\|\bar{\mathbf{H}}_c - \bar{\mathbf{H}}_{c'}\| \ge \Delta - 2\eta$. Thus
\[
\mathrm{FR}(c,c') \ge \frac{(\Delta - 2\eta)^2}{8\eta^2}.
\]

\textit{(b).}
If $i,j \in C_s$, then $\mathbf{H}_i - \mathbf{H}_j = \boldsymbol{\xi}_i - \boldsymbol{\xi}_j$,
so $\|\mathbf{H}_i - \mathbf{H}_j\| \le 2\eta$.
Since $f_c(\mathbf{h}) = \mathbf{w}_c^\top \mathbf{h} + b_c$ classifies $i$ with label $a$
correctly with margin $\gamma$, we have in particular:
\[
(\mathbf{w}_a - \mathbf{w}_b)^\top \mathbf{H}_i + (b_a - b_b) \ge \gamma.
\]
Since $f_c$ classifies $j$ with label $b$ correctly with margin $\gamma$, we have:
\[
(\mathbf{w}_b - \mathbf{w}_a)^\top \mathbf{H}_j + (b_b - b_a) \ge \gamma.
\]
Adding these two inequalities, the bias terms cancel:
\[
(\mathbf{w}_a - \mathbf{w}_b)^\top (\mathbf{H}_i - \mathbf{H}_j) \ge 2\gamma.
\]
By Cauchy--Schwarz,
\[
2\gamma \le \|\mathbf{w}_a - \mathbf{w}_b\| \cdot \|\mathbf{H}_i - \mathbf{H}_j\| \le 2\eta\|\mathbf{w}_a - \mathbf{w}_b\|,
\]
which gives $\|\mathbf{w}_a - \mathbf{w}_b\| \ge \gamma/\eta$.
\end{proof}

\section{Proofs of CASP Theoretical Properties}
\label{app:casp_proofs}

The following propositions establish three key properties of the CASP 
parameterisation introduced in Section~\ref{sec:casp}: that it recovers 
standard propagation as a special case, that the propagation ratio $r$ 
provides monotone and interpretable control over the intra/inter balance, 
and that the expected aggregation mass is preserved for any $r>0$.

\begin{proposition}[Baseline recovery]
\label{prop:recovery}
At $r=1$, $\alpha_{\intra}=\alpha_{\inter}=1$, so CASP reduces to standard 
propagation.
\end{proposition}
\begin{proof}
Substituting $r=1$ into~\eqref{eq:alphas} gives
$\alpha_{\intra}=\alpha_{\inter}=1$, so
$\alpha_{\intra}\Ahat_{\intra}+\alpha_{\inter}\Ahat_{\inter}=\Ahat$.
\end{proof}

\begin{proposition}[Monotone interpolation]
\label{prop:monotone_alpha}
$\alpha_{\intra}$ increases and $\alpha_{\inter}$ decreases monotonically in 
$r$, with $\alpha_{\intra}\to 0$, $\alpha_{\inter}\to 1/(1-p)$ as $r\to 0$,
and $\alpha_{\intra}\to 1/p$, $\alpha_{\inter}\to 0$ as $r\to\infty$.
\end{proposition}
\begin{proof}
Differentiating Equation~\eqref{eq:alphas} gives
$\partial\alpha_{\intra}/\partial r=(1-p)/(rp+(1-p))^2>0$ and
$\partial\alpha_{\inter}/\partial r=-p/(rp+(1-p))^2<0$.
The limits follow directly from~\eqref{eq:alphas}.
\end{proof}

\begin{proposition}[Expected mass preservation]
\label{prop:mass}
Under the balanced SBM, for any $r>0$,
$
\mathbb{E}_i\!\left[
\alpha_{\intra}\frac{k_i^{\intra}}{d_i}
+
\alpha_{\inter}\frac{k_i^{\inter}}{d_i}
\right]=1.
$
\end{proposition}
\begin{proof}
Under the balanced SBM,
$\mathbb{E}[k_i^{\intra}/d_i]=p$ and $\mathbb{E}[k_i^{\inter}/d_i]=1-p$.
Substituting~\eqref{eq:alphas},
\[
\alpha_{\intra}p+\alpha_{\inter}(1-p)
=\frac{rp}{rp+(1-p)}+\frac{1-p}{rp+(1-p)}=1. \qedhere
\]
\end{proof}

\section{Experimental Design} \label{sec:exp_des}

\subsection{Dataset Statistics}
\label{app:datasets}

Table~\ref{tab:dataset_stats} reports statistics for all eleven node 
classification benchmarks used in our experiments, spanning homophilic 
and heterophilic settings. 

\begin{table}[htbp]
\centering
\small
\caption{Dataset statistics.}
\label{tab:dataset_stats}
\begin{tabular}{llrrrr}
\toprule
Type & Dataset & \# Nodes & \# Edges & \# Features & \# Classes \\
\midrule
\multirow{8}{*}{Homophilic}
& Cora            & 2,708   & 5,429     & 1,433 & 7  \\
& Cora Full       & 19,793  & 126,842   & 8,710 & 70 \\
& Citeseer        & 3,312   & 4,732     & 3,703 & 6  \\
& Pubmed          & 19,717  & 44,338    & 500   & 3  \\
& DBLP            & 17,716  & 105,734   & 1,639 & 4  \\
& Amazon Photo    & 7,650   & 119,081   & 745   & 8  \\
& Coauthor-CS     & 18,333  & 163,788   & 6,805 & 15 \\
& ogbn-arxiv      & 169,343 & 1,166,243 & 128   & 40 \\
\midrule
\multirow{9}{*}{Heterophilic}
& Cornell         & 183     & 295       & 1,703 & 5  \\
& Wisconsin       & 251     & 499       & 1,703 & 5  \\
& Texas           & 183     & 309       & 1,703 & 5  \\
& Film            & 7,600   & 33,544    & 932   & 5  \\
& Minesweeper     & 10,000  & 39,402    & 7     & 2  \\
& Tolokers        & 11,758  & 519,000   & 10    & 2  \\
& Questions       & 48,921  & 153,540   & 301   & 2  \\
& Amazon-Ratings  & 24,492  & 93,050    & 300   & 5  \\
& Roman-Empire    & 22,662  & 32,927    & 300   & 18 \\
\bottomrule
\end{tabular}
\end{table}

\subsection{Model Hyperparamters}
\label{app:hyperparams}

We perform a grid search 
over the following hyperparameters: learning rate $\in \{0.005, 0.01, 0.1\}$, 
weight decay $\in \{5\times10^{-4}, 5\times10^{-3}\}$, dropout 
$\in \{0.1, 0.4, 0.8\}$, and hidden dimension $\in \{64, 128, 256\}$. 
We train for up to $1000$ epochs with early stopping patience of $200$ 
epochs based on validation loss. The number of layers is varied up to $4$. We use a resolution parameter of 1.0 for the Louvain algorithm across all datasets. All models are trained using the Adam optimiser \cite{kingma2014adam}.

\subsection{Implementation Details}
\label{app:impl}

All experiments were performed on a Linux-based server equipped with an NVIDIA A100 GPU. Our implementation uses the following Python libraries: PyTorch (v2.3.1), torchvision (v0.18.1), torchaudio (v2.3.1), torch-geometric (v2.7.0), torch-cluster (v1.6.3), torch-scatter (v2.0.9), and torch-sparse (v0.6.18).

\section{Additional Experiments}
\label{app:additional}

\subsection{Ablation Studies on ECI}
\label{app:eci_exp}

\begin{figure}[htbp]
\centering
\includegraphics[width=\linewidth]{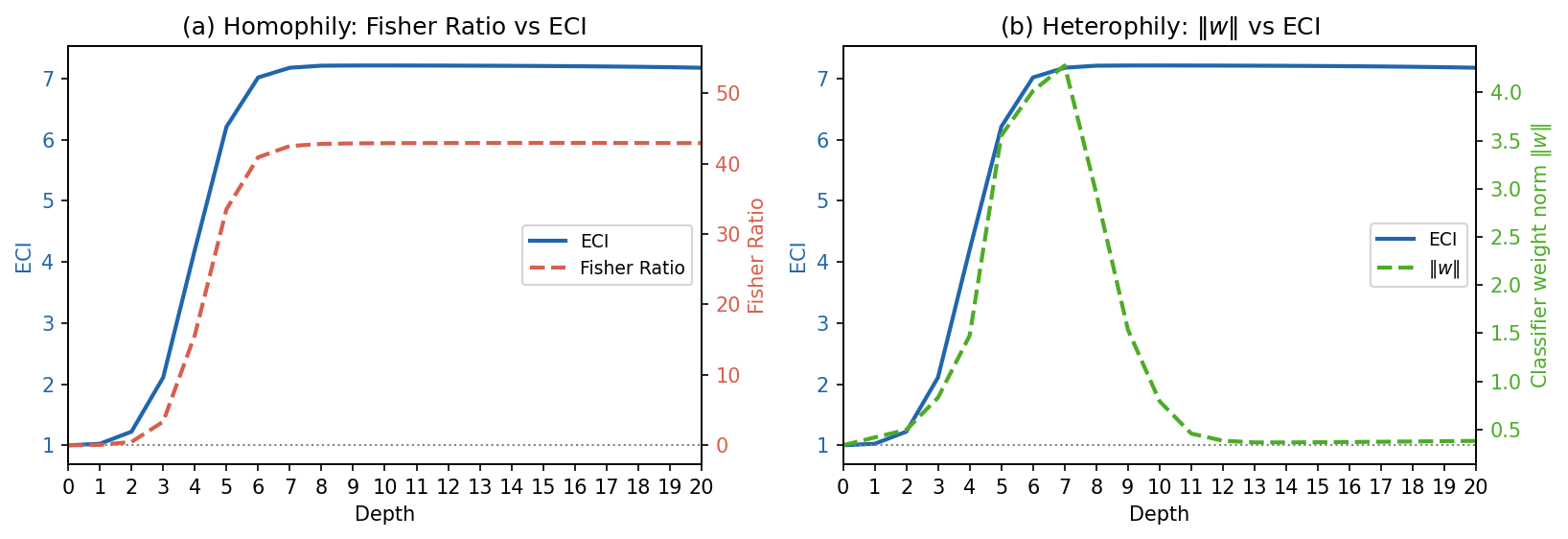}
\caption{Validation of Theorem~\ref{thm:task} on SBM graphs
($n{=}1500$, $m{=}5$, $\mathbb{E}[\mathrm{deg}]{=}20$).}
\label{fig:theorem-validation}
\end{figure}

\subsubsection{Emperical Validation of Theorem~\ref{thm:task}}
\label{app:thm_validation}

Figure~\ref{fig:theorem-validation} measures the quantities predicted by 
Theorem~\ref{thm:task} as a function of propagation depth on SBM graphs. 
Under homophily, the Fisher Ratio rises alongside ECI and plateaus at the 
same depth, confirming part~(a): as intra-community representations collapse, 
inter-class separation grows unboundedly relative to intra-class variance. 
Under heterophily, the binary linear SVM weight norm $\|w\|^{(k)}$ peaks 
precisely at peak ECI (depth $\approx 7$), confirming part~(b): 
classification is hardest exactly when the echo chamber is deepest, as 
same-community nodes of different classes become indistinguishable and 
maintaining any fixed margin $\gamma > 0$ requires $\|w\| \geq \gamma/(2\eta)$. 
The subsequent decay of $\|w\|$ beyond the echo chamber peak reflects further 
representation collapse toward a region where no linear separator exists.

\subsubsection{Robustness to Community Detection}
\label{app:comm_detection}

Figure~\ref{fig:comm-detection} tests sensitivity of ECI to the choice 
of community partition, comparing four algorithms: Louvain~\cite{blondel2008fast}, 
Leiden~\cite{traag2019louvain}, spectral clustering~\cite{von2007tutorial}, 
and a random baseline. On SBM,  Louvain, Leiden, and spectral clustering algorithms recover ECI profiles 
that are virtually identical to those obtained using the known community 
memberships (i.e., ground truth) used to generate the graph, confirming that community detection 
introduces no distortion on well-structured synthetic graphs. On 
Coauthor-CS, Louvain and Leiden yield nearly identical profiles while 
spectral clustering yields a moderately lower peak. Random partitions 
remain flat at unity throughout, confirming that the echo chamber signal 
is driven by the underlying graph structure rather than the choice of 
community detection algorithm. Louvain algorithm is used as the default algorithm throughout our experiments
due to its computational efficiency.

\begin{figure}[htbp]
\centering
\includegraphics[width=0.9\linewidth]{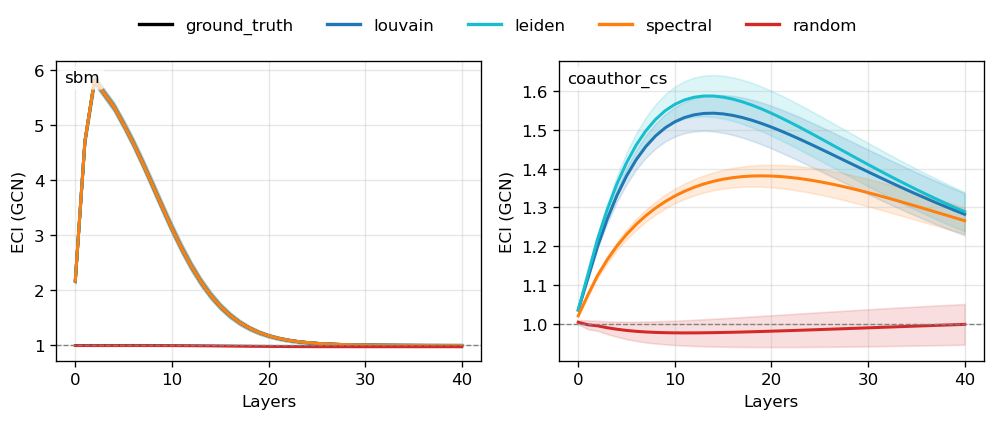}
\caption{ECI under different community partitions (GCN) on SBM and 
Coauthor-CS. For SBM, ground-truth communities are the planted blocks.}
\label{fig:comm-detection}
\end{figure}

\begin{figure}[htbp]
\centering
\includegraphics[width=0.95\linewidth]{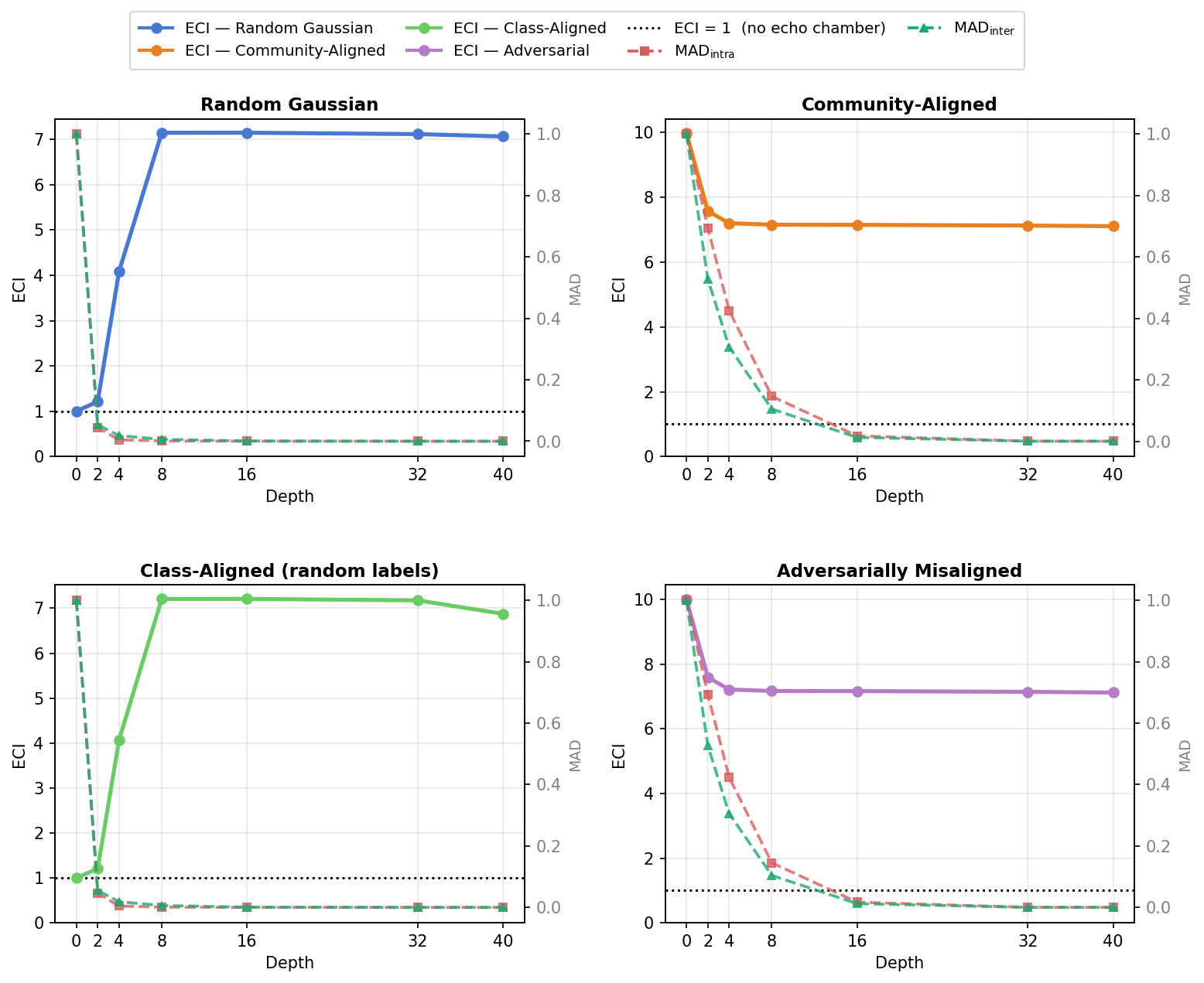}
\caption{ECI under four initial feature conditions on a fixed SBM graph 
($n{=}1500$, $m{=}5$, $p_{\mathrm{intra}}{=}0.08$, 
$p_{\mathrm{inter}}{=}0.005$). $\MAD_{\mathrm{intra}}$ decays faster 
than $\MAD_{\mathrm{inter}}$ across all conditions, explaining the 
persistent elevation of ECI even as both approach zero.}
\label{fig:eci_feature_sensitivity}
\end{figure}

\subsubsection{Effect of Initial Node Features on ECI}
\label{app:feature_sensitivity}

To determine whether the echo chamber effect is driven by graph topology 
or by the initial node feature distribution, we measure ECI under pure 
propagation on a fixed SBM graph across four feature conditions: random 
Gaussian (no structure); community-aligned (features drawn around community 
centroids); class-aligned (features encode randomly assigned labels, 
independent of the community); and adversarially misaligned (each node 
receives the centroid of a different community). 
Figure~\ref{fig:eci_feature_sensitivity} shows ECI as a function of 
propagation depth for each condition. Despite starting from very different 
initial values, all four conditions converge to the same asymptotic plateau. 
This is explained by the MAD traces: $\MAD_{\intra}$ decays substantially 
faster than $\MAD_{\inter}$ across all conditions, so their ratio ECI 
remains elevated even as both approach zero. Crucially, even adversarially 
misaligned features are dominated by graph topology within a limited number 
of propagation steps, demonstrating that the asymptotic echo chamber region 
is primarily determined by graph modularity rather than by the initial 
feature distribution. Initial features affect only the transient trajectory 
and the depth at which the echo chamber is entered, not the long-run steady 
state.

\subsubsection{Diagnostic Gap: ECI vs.\ Existing Metrics}
\label{app:diag_gap}

To demonstrate that ECI captures structural information that existing metrics 
miss, we conduct two complementary experiments on synthetic SBM graphs 
($n=400$, $m=5$, $p_{\text{intra}}=0.10$).

\textbf{Modularity sweep.} Figure~\ref{fig:eci_differentiated} sweeps 
$\varepsilon = p_{\text{inter}}/p_{\text{intra}} \in \{0.4, 0.1, 0.02, 
0.005\}$, progressively strengthening community structure. As $\varepsilon 
\to 0$, MAD and Dirichlet energy decay monotonically to zero, falsely 
signalling global oversmoothing. ECI and GDR both remain persistently 
elevated, correctly detecting that inter-community separation is intact, which is 
the diagnostic gap established in Proposition \ref{prop:dirichlet}. Embedding modularity 
saturates early and fails to reflect the progressive widening of the echo 
chamber window as $\varepsilon$ decreases.

\textbf{Label--community alignment.} Figure~\ref{fig:eci_vs_gdr} fixes 
$\varepsilon = 0.05$ and varies the fraction of nodes whose label aligns 
with their community from $100\%$ (homophilic) to $25\%$ (near-random). 
ECI remains invariant across all alignment levels since it depends solely 
on graph topology. GDR, which requires ground-truth labels, collapses 
toward unity as alignment decreases, becoming uninformative under 
heterophily. This demonstrates the fundamental limitation of GDR as a 
structural diagnostic: it conflates label structure with community structure 
and fails precisely in the settings where the echo chamber is most harmful. 
ECI avoids this by stratifying distances by community membership rather 
than class membership, making it reliably 
detect the echo chamber across all graph and label configurations.

\begin{figure}[htbp]
    \centering
    \includegraphics[width=\linewidth]{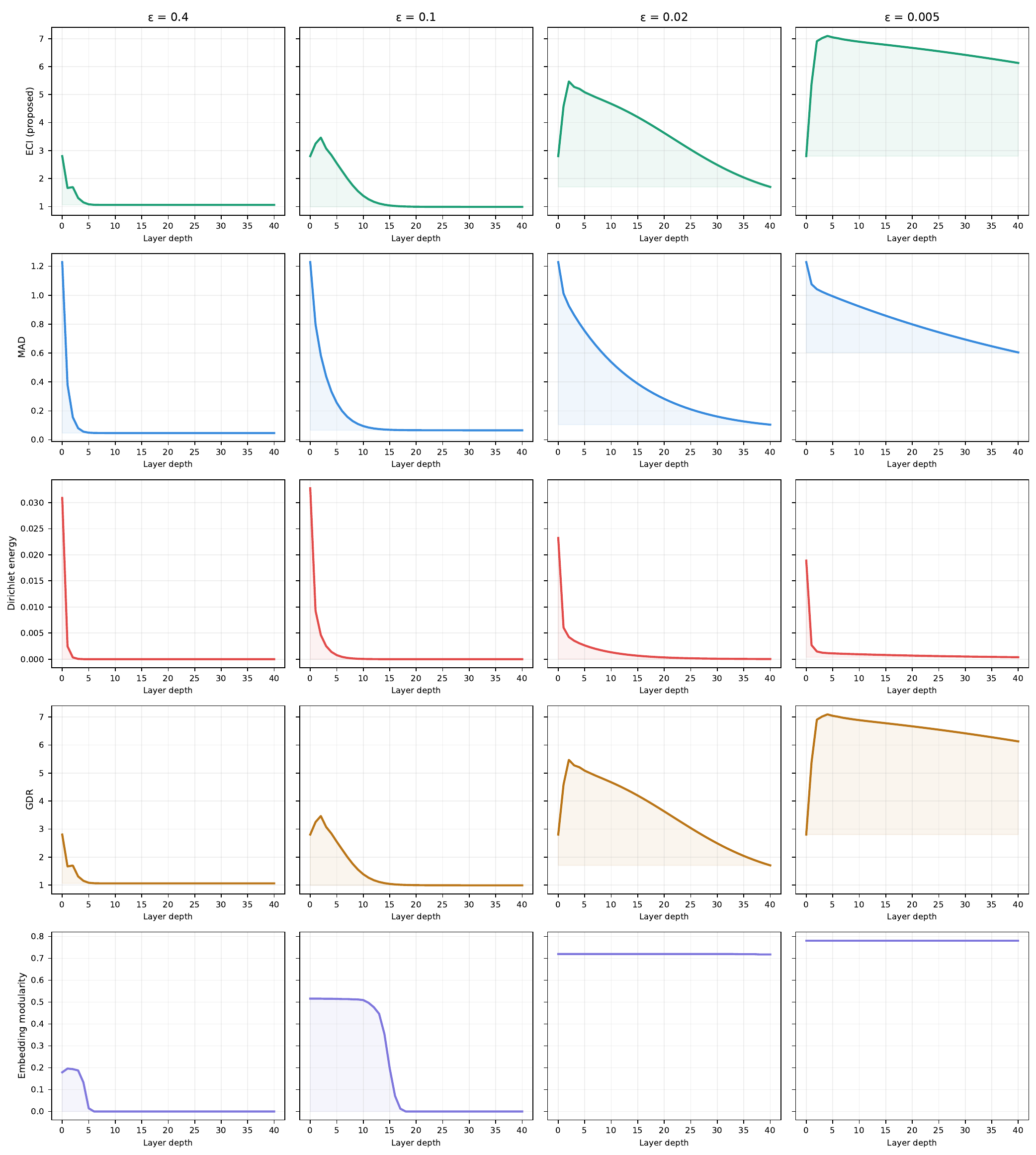}
    \caption{Layer-wise metrics on a synthetic SBM graph across four values 
    of $\varepsilon$.}
    \label{fig:eci_differentiated}
\end{figure}

\begin{figure}[htbp]
    \centering
    \includegraphics[width=\linewidth]{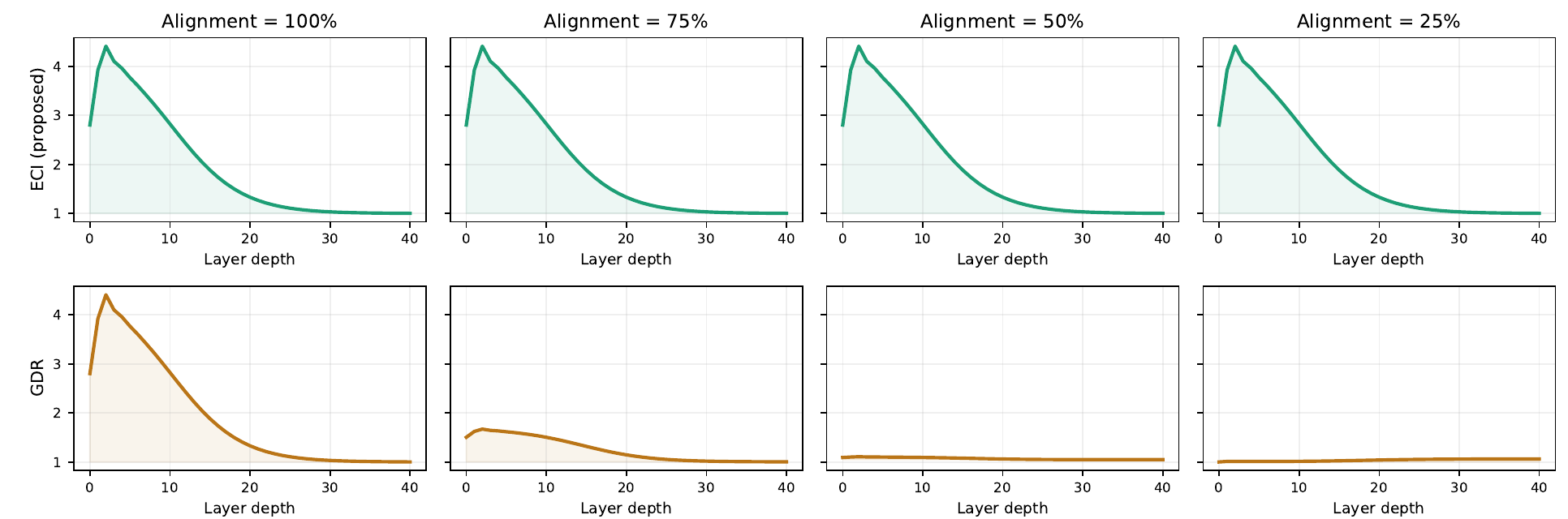}
    \caption{ECI vs.\ GDR under varying label--community alignment 
    ($\varepsilon=0.05$).}
    \label{fig:eci_vs_gdr}
\end{figure}

\clearpage

\subsubsection{Empirical Validation of Theoretical Assumptions}
\label{app:assumption-validation}

\paragraph{Representation concentration.}
Proposition~\ref{prop:dirichlet} assumes
$\bH_i=\bar{\bh}_{c(i)}+\boldsymbol{\xi}_i$ with
$\|\boldsymbol{\xi}_i\|\leq\eta<\Delta/2$, where $\Delta$ is the separation
between community centroids. We evaluate this condition using
$2\eta_{95}/\Delta$, where $\eta_{95}$ is the 95th percentile of
node-to-centroid distances. Values below one indicate that the condition
holds.

\begin{table*}[htbp]
\centering
\caption{Empirical evaluation of the concentration assumption. ``Holding''
is the percentage of community pairs satisfying $\eta_{95}<\Delta/2$.
Results are averaged over five seeds. $\dagger$ denotes GCN runs that fail
to train, with accuracy between 0.21 and 0.23.}
\label{tab:concentration}
\small
\setlength{\tabcolsep}{4pt}
\begin{tabular}{lc cc cc cc}
\toprule
& & \multicolumn{2}{c}{SBM, i.i.d.}
& \multicolumn{2}{c}{SBM, aligned}
& \multicolumn{2}{c}{Citeseer} \\
\cmidrule(lr){3-4}\cmidrule(lr){5-6}\cmidrule(lr){7-8}
Model & Layer
& Median & Holding
& Median & Holding
& Median & Holding \\
\midrule
GCN              & 2  & 0.74 & 100\% & 0.40 & 100\% & 1.96 & 9.2\%  \\
GCN              & 4  & 0.48 & 100\% & 0.51 & 100\% & 2.07 & 9.0\%  \\
GCN$^\dagger$    & 8  & 37.9 & 0\%   & 37.3 & 0\%   & 32.4 & 0\%    \\
GCN$^\dagger$    & 16 & 37.5 & 0\%   & 37.4 & 0\%   & 34.9 & 0\%    \\
\midrule
APPNP            & 2  & 0.85 & 100\% & 0.34 & 100\% & 1.77 & 12.0\% \\
APPNP            & 4  & 1.01 & 54\%  & 0.34 & 100\% & 1.67 & 15.4\% \\
APPNP            & 8  & 1.26 & 0\%   & 0.34 & 100\% & 1.57 & 18.6\% \\
APPNP            & 16 & 1.38 & 0\%   & 0.35 & 100\% & 1.52 & 21.2\% \\
\midrule
GCNII            & 2  & 0.60 & 100\% & 0.38 & 100\% & 1.90 & 8.8\%  \\
GCNII            & 4  & 0.53 & 100\% & 0.39 & 100\% & 1.83 & 12.0\% \\
GCNII            & 8  & 0.76 & 94\%  & 0.36 & 100\% & 1.77 & 12.0\% \\
GCNII            & 16 & 2.21 & 0\%   & 0.29 & 100\% & 1.79 & 11.0\% \\
\bottomrule
\end{tabular}
\end{table*}

The condition holds broadly in the SBM setting for which the theory is
formulated and more selectively on Citeseer. With community-aligned SBM
features, it holds at every tested depth except when GCN fails to train.
The condition is therefore sufficient rather than universal in trained
networks.

For a finer depth-wise analysis, Table~\ref{tab:concentration-depth} reports
the same quantity for a 16-layer APPNP on the SBM. The condition holds most
frequently at layers 2--4, where ECI is also highest, connecting the
theoretical regime to the strongest observed echo-chamber region.

\begin{table}[htbp]
\centering
\caption{Layer-wise concentration and ECI for 16-layer APPNP on an SBM with
i.i.d.\ features, averaged over five seeds.}
\label{tab:concentration-depth}
\small
\begin{tabular}{cccc}
\toprule
Layer & Median $2\eta_{95}/\Delta$ & Pairs holding & ECI \\
\midrule
0  & 4.35 & 0\%  & 1.15 \\
1  & 1.06 & 30\% & 2.49 \\
2  & 0.90 & 84\% & 2.96 \\
3  & 0.89 & 92\% & 3.01 \\
4  & 0.98 & 56\% & 2.75 \\
5  & 1.06 & 24\% & 2.57 \\
6  & 1.14 & 6\%  & 2.42 \\
8  & 1.25 & 0\%  & 2.24 \\
16 & 1.38 & 0\%  & 2.08 \\
\bottomrule
\end{tabular}
\end{table}

\paragraph{Mass preservation on real graphs.}
Proposition~\ref{prop:mass} establishes expected mass
preservation under a balanced SBM. To evaluate its applicability to real
graphs, we measure each node's intra-community degree fraction
$p_i=k_i^{\mathrm{intra}}/d_i$ and aggregation mass
$m_i=\alpha_{\intra}p_i+\alpha_{\inter}(1-p_i)$ under the Louvain partition.

\begin{table}[htbp]
\centering
\caption{Per-node intra-community degree fractions and aggregation mass on
real graphs. ``Within 10\%'' is the fraction of nodes satisfying
$|m_i-1|<0.1$.}
\label{tab:mass-preservation}
\small
\begin{tabular}{lcccc}
\toprule
Dataset & Mean $p_i$ & Std.\ $p_i$ & Mean $m_i$ & Within 10\% \\
\midrule
Roman-Empire   & 0.995 & 0.045 & 1.000 & 98\% \\
Amazon-Ratings & 0.979 & 0.074 & 0.999 & 92\% \\
\bottomrule
\end{tabular}
\end{table}

The mean aggregation mass is approximately one on both datasets, and most
nodes remain within 10\% of unit mass. Thus, the mass-preserving
parameterisation remains a close empirical approximation beyond the balanced
SBM setting.

\subsection{Ablation Studies on CASP}
\label{app:casp_ablation}

\subsubsection{Component Analysis of CASP}
\label{app:casp_components}

Table~\ref{tab:ablation} isolates the contribution of each CASP 
components across a subset of homophilic and 
heterophilic datasets. Intra-only propagation 
improves over baseline on homophilic datasets but underperforms on 
heterophilic ones, and vice versa for inter-only propagation, 
confirming that neither direction alone is sufficient. The 
non-mass-preserving variant performs comparably to CASP Full, 
indicating the mass-preserving parameterisation contributes modest 
but consistent gains by preventing signal scale distortion. Full CASP 
achieves the best accuracy in all datasets. Figure~\ref{fig:r_sweep} shows accuracy as a function of fixed $r$ on 
Citeseer and Film datasets. On Citeseer, accuracy peaks when $r > 1$, confirming 
that stronger intra-community aggregation improves homophilic 
classification. On Film, the peak occurs at $r < 1$, confirming 
that cross-community mixing is beneficial under heterophily. Both 
optima are consistent with Theorem~1 and validate the data-driven 
calibration of $r$.
\begin{table}[htbp]
\begin{minipage}[htbp]{0.50\linewidth}
\vspace{0pt}
\centering
\caption{Ablation of CASP components}
\label{tab:ablation}
\setlength{\tabcolsep}{4pt}
\renewcommand{\arraystretch}{1.1}
\resizebox{\linewidth}{!}{%
\begin{tabular}{lcccc}
\toprule
& \multicolumn{2}{c}{\textbf{Homophilic}} 
& \multicolumn{2}{c}{\textbf{Heterophilic}} \\
\cmidrule(lr){2-3} \cmidrule(lr){4-5}
\textbf{Method} & Citeseer & Pubmed & Film & Amazon-Ratings \\
\midrule
Baseline (GCN)        
    & 76.68\tiny{$\pm$1.64} 
    & 86.74\tiny{$\pm$0.47} 
    & 30.26\tiny{$\pm$0.79} 
    & 37.99\tiny{$\pm$0.61} \\
Intra-only propagation           
    & 78.57\tiny{$\pm$2.21} 
    & 89.14\tiny{$\pm$0.46} 
    & 34.80\tiny{$\pm$2.77} 
    & 46.42\tiny{$\pm$1.32} \\
Inter-only propagation          
    & 77.12\tiny{$\pm$1.86} 
    & 87.83\tiny{$\pm$0.35} 
    & 36.09\tiny{$\pm$0.98} 
    & 48.36\tiny{$\pm$1.34} \\
Non-mass-preserving   
    & 79.03\tiny{$\pm$1.08} 
    & 89.42\tiny{$\pm$0.80} 
    & 37.00\tiny{$\pm$1.84} 
    & 49.19\tiny{$\pm$1.44} \\
\midrule
CASP Full             
    & \textbf{79.36}\tiny{$\pm$1.89} 
    & \textbf{89.68}\tiny{$\pm$0.50} 
    & \textbf{37.23}\tiny{$\pm$2.17} 
    & \textbf{49.57}\tiny{$\pm$0.95} \\
\bottomrule
\end{tabular}%
}
\end{minipage}
\hfill
\begin{minipage}[htbp]{0.47\linewidth}
\vspace{0pt}
\centering
\includegraphics[width=\linewidth, height=21ex]{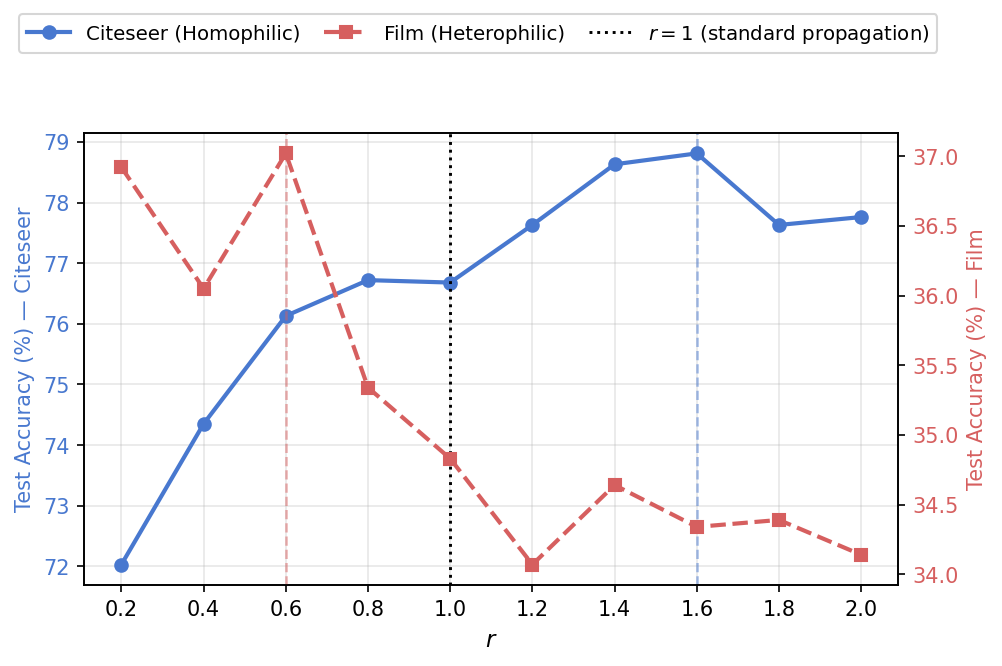}
\captionof{figure}{Accuracy vs fixed $r$}
\label{fig:r_sweep}
\end{minipage}
\end{table}

\subsubsection{Robustness to Data Splits and Community Detection Algorithm}
\label{app:split_robustness}

Table~\ref{tab:casp_split_robustness} and 
Table~\ref{tab:casp_comm_robustness} evaluate CASP under two robustness 
conditions. Across three train/validation/test splits ranging from very 
sparse (2.5/2.5/95) to standard (60/20/20) label regions, CASP 
consistently improves over both GCN and APPNP on Pubmed and Film, 
confirming that the data-driven calibration of $r$ remains effective 
regardless of label availability. Substituting the default Louvain 
partition with Leiden, spectral clustering, or label propagation yields 
comparable accuracy, confirming that CASP is robust to the choice of 
community detection algorithm.

\begin{table}[htbp]
\centering
\footnotesize
\caption{Robustness of CASP to data splits on Pubmed (homophilic) and Film (heterophilic). Best results are highlighted in \textbf{bold}.}
\label{tab:casp_split_robustness}

\begin{tabular*}{\textwidth}{@{\extracolsep{\fill}} l cccccc}
\toprule
& \multicolumn{3}{c}{\textbf{Pubmed}} 
& \multicolumn{3}{c}{\textbf{Film}} \\
\cmidrule(lr){2-4} \cmidrule(lr){5-7}
\textbf{Method} 
& 2.5/2.5/95 & 48/32/20 & 60/20/20 
& 2.5/2.5/95 & 48/32/20 & 60/20/20 \\
\midrule
GCN           
& 78.81 $\pm$ 0.24 & 87.38 $\pm$ 0.66 & 86.74 $\pm$ 0.47
& 22.74 $\pm$ 2.37 & 30.59 $\pm$ 0.23 & 30.26 $\pm$ 0.79 \\
GCN + CASP    
& \textbf{82.53 $\pm$ 0.78} & \textbf{89.70 $\pm$ 0.33} & \textbf{89.68 $\pm$ 0.50}
& \textbf{24.76 $\pm$ 1.24} & \textbf{32.54 $\pm$ 3.36} & \textbf{37.23 $\pm$ 2.17} \\
APPNP         
& 79.97 $\pm$ 0.28 & 85.02 $\pm$ 0.09 & 83.37 $\pm$ 0.42
& 29.74 $\pm$ 1.04 & 34.86 $\pm$ 1.32 & 36.44 $\pm$ 1.78 \\
APPNP + CASP  
& \textbf{83.18 $\pm$ 1.06} & \textbf{88.95 $\pm$ 0.48} & \textbf{88.76 $\pm$ 0.35}
& \textbf{29.91 $\pm$ 0.87} & \textbf{37.58 $\pm$ 1.67} & \textbf{38.15 $\pm$ 1.59} \\
\bottomrule
\end{tabular*}
\end{table}

\begin{table}[htbp]
\centering
\footnotesize
\caption{Robustness of CASP to community detection algorithm on Pubmed (homophilic) and Film (heterophilic).}
\label{tab:casp_comm_robustness}

\begin{tabular*}{0.7\textwidth}{@{\extracolsep{\fill}} l cc}
\toprule
\textbf{Algorithm} 
& \textbf{Pubmed} 
& \textbf{Film}  \\
\midrule
GCN (baseline)      
& 86.74 $\pm$ 0.47 & 30.26 $\pm$ 0.79 \\
Louvain             
& 89.68 $\pm$ 0.50 & 37.23 $\pm$ 2.17 \\
Leiden              
& 90.48 $\pm$ 0.44 & 34.50 $\pm$ 1.57 \\
Spectral clustering 
& 90.27 $\pm$ 0.52 & 37.08 $\pm$ 2.60 \\
Label propagation   
& 90.50 $\pm$ 0.57 & 36.83 $\pm$ 2.37 \\
\bottomrule
\end{tabular*}
\end{table}

\subsubsection{Robustness to Noisy Community Partitions}
\label{app:noisy_partitions}

Figure~\ref{fig:partition_robustness} tests sensitivity of ECI and 
CASP-GCN to errors in community detection by randomly reassigning a 
fraction $\rho \in \{0.0, 0.1, \ldots, 0.5\}$ of nodes to incorrect 
communities. ECI degrades gracefully with $\rho$: even at $\rho{=}0.5$, 
the echo chamber signal remains well above unity, confirming the diagnostic 
is robust to partition noise. Node classification accuracy on Pubmed 
is flat across all $\rho$, with no statistically meaningful 
degradation. The learned $r$ remains consistently above unity throughout, 
indicating that CASP correctly identifies the homophilic propagation 
direction even under substantial partition corruption, with noise 
attenuating the magnitude of $r$ but never reversing its sign.

\begin{figure}[htbp]
\centering
\begin{subfigure}[t]{0.38\linewidth}
    \includegraphics[width=\linewidth]{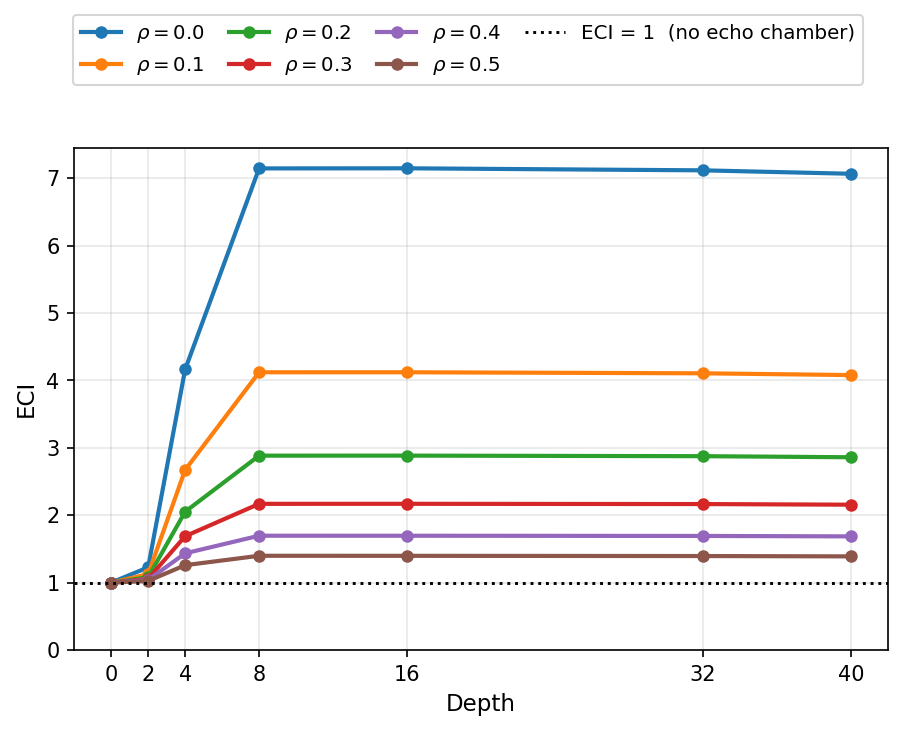}
    \caption{ECI vs depth under noisy partitions on SBM.}
    \label{fig:partition_robustness_eci}
\end{subfigure}
\hfill
\begin{subfigure}[t]{0.58\linewidth}
    \includegraphics[width=\linewidth]{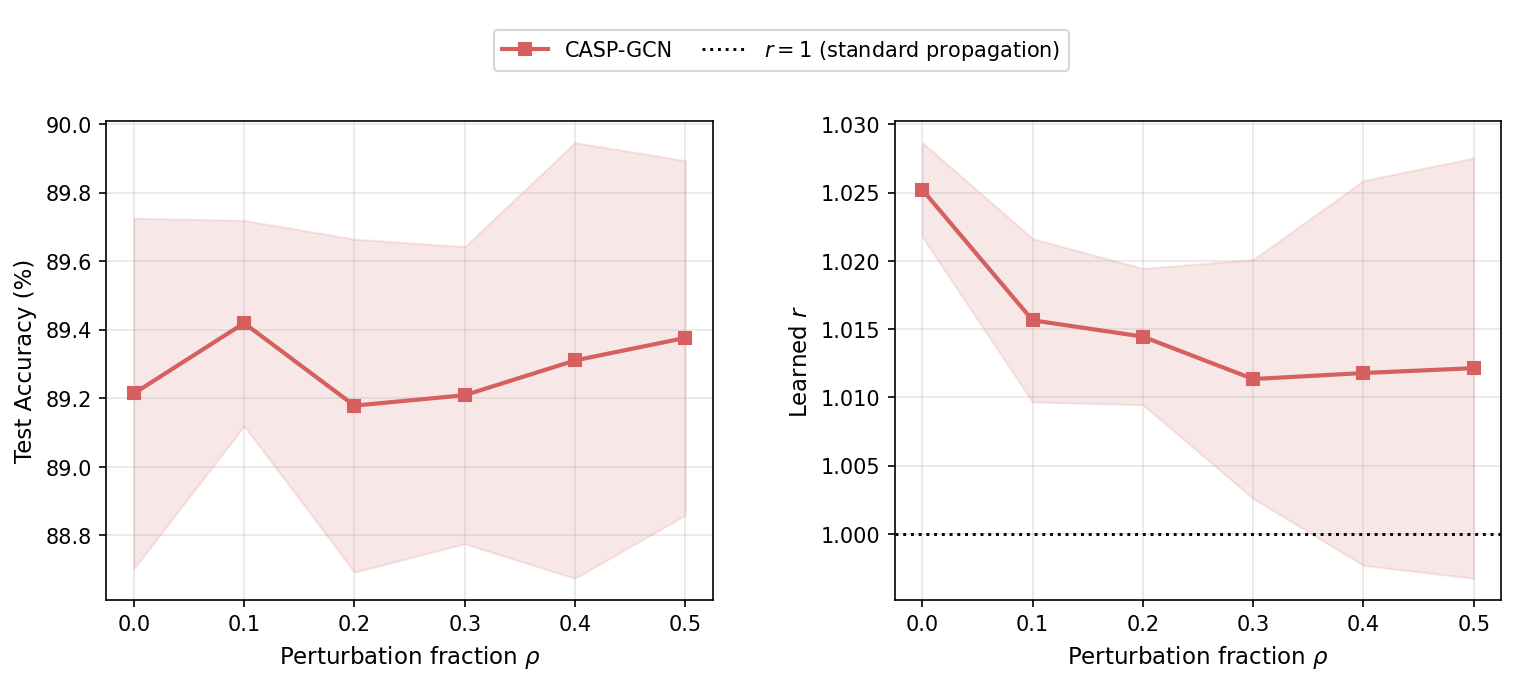}
    \caption{CASP-GCN accuracy and learned $r$ on Pubmed.}
    \label{fig:partition_robustness_acc_r}
\end{subfigure}
\caption{Robustness to noisy community partitions. ECI remains above 
unity and accuracy is stable even at $\rho{=}0.5$.}
\label{fig:partition_robustness}
\end{figure}

\begin{figure}[htbp]
\centering
\includegraphics[width=\linewidth]{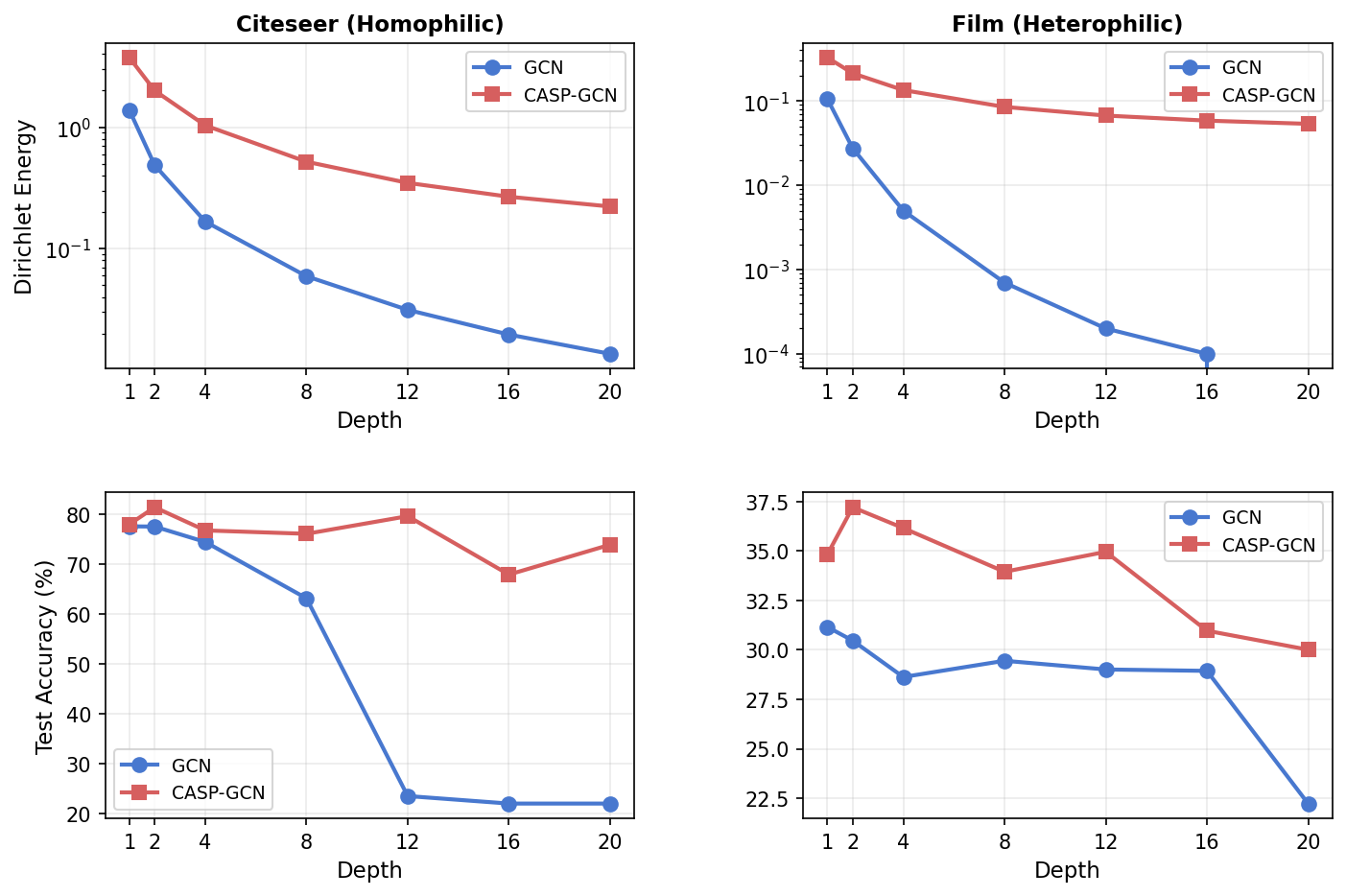}
\caption{Dynamic depth analysis: Dirichlet energy (top) and test accuracy 
(bottom) vs.\ depth for GCN and CASP-GCN on Citeseer (homophilic) and 
Film (heterophilic) datasets.}
\label{fig:depth_analysis}
\end{figure}

\subsubsection{Depth Analysis and Oversmoothing Robustness}
\label{app:depth_analysis}

Figure~\ref{fig:depth_analysis} evaluates Dirichlet energy and test 
accuracy as a function of depth on homophilic and heterophilic datasets. 
GCN accuracy degrades sharply at the depths where Dirichlet energy 
approaches zero, exhibiting the standard oversmoothing failure. CASP-GCN 
sustains significantly higher Dirichlet energy throughout, and its accuracy 
remains stable well beyond the depths at which GCN collapses. This confirms 
that CASP reshapes the oversmoothing window rather than merely delaying it: 
by rebalancing intra- and inter-community propagation according to the 
graph's label structure, CASP maintains useful representations at depths 
where standard GCN has already converged to a trivial solution.

\subsubsection{Node Classification Performance with Existing Baselines}
\label{app:baselines}

Table~\ref{tab:results} compares CASP with existing node-classification GNNs
on homophilic (Cora, Citeseer, and Pubmed) and heterophilic (Film, Cornell,
Wisconsin, and Texas) benchmarks. The baselines include general-purpose GNNs:
GCN~\cite{kipf2017semi}, GAT~\cite{velivckovic2018graph},
GraphSAGE~\cite{hamilton2017inductive}, SGC~\cite{wu2019simplifying}, and
GCN-Cheby~\cite{defferrard2016convolutional}; and heterophily-oriented
architectures: MixHop~\cite{abu2019mixhop},
{GCNII~\cite{chen2020simple},
WRGAT~\cite{suresh2021breaking}, ACM-GCN~\cite{luan2021heterophily},
H$_2$GCN~\cite{zhu2020beyond},
GPR-GNN~\cite{chien2021adaptive},
GGCN~\cite{yan2022two}, LINKX~\cite{lim2021large},
GloGNN~\cite{li2022finding}, DirGNN~\cite{rossi2024edge},
{HiGNN~\cite{zheng2025learn}, and
BEC-GNN~\cite{hevapathige2025depth}.

CASP improves both evaluated backbones across all seven datasets, although the
magnitude of improvement depends on the backbone. GCN+CASP remains below
stronger heterophily-oriented methods on Cornell, Wisconsin, and Texas,
whereas DirGNN+CASP achieves the best results. Overall, CASP provides
consistent backbone-relative improvements and is most effective when combined
with an architecture suited to the underlying graph structure.

\begin{table*}[htbp]
\centering
\caption{Comparison of node classification accuracy (\%)
using ten random 60/20/20 train/validation/test splits. Values are means and
standard deviations across the ten runs. Best results are highlighted in
\textbf{bold}. Baseline results are taken from
\cite{zheng2025learn}.}
\label{tab:results}
\resizebox{\textwidth}{!}{%
\begin{tabular}{lccccccc}
\toprule
Methods & Cora & Citeseer & Pubmed & Film & Cornell & Wisconsin & Texas \\
\midrule
GCN
    & 86.90 $\pm$ 1.09 & 76.68 $\pm$ 1.64 & 87.38 $\pm$ 0.66
    & 30.26 $\pm$ 0.79 & 57.03 $\pm$ 4.67 & 59.80 $\pm$ 6.99
    & 59.46 $\pm$ 5.25 \\
GAT
    & 86.06 $\pm$ 1.31 & 75.46 $\pm$ 1.72 & 87.62 $\pm$ 0.42
    & 26.28 $\pm$ 1.73 & 58.92 $\pm$ 3.32 & 55.29 $\pm$ 8.71
    & 58.38 $\pm$ 4.45 \\
GraphSAGE
    & 80.44 $\pm$ 1.90 & 76.04 $\pm$ 1.30 & 88.45 $\pm$ 0.50
    & 34.23 $\pm$ 0.99 & 75.95 $\pm$ 5.01 & 81.18 $\pm$ 5.56
    & 82.43 $\pm$ 6.14 \\
SGC
    & 84.97 $\pm$ 1.95 & 75.66 $\pm$ 1.37 & 87.15 $\pm$ 0.47
    & 25.83 $\pm$ 1.09 & 55.41 $\pm$ 5.29 & 57.84 $\pm$ 4.83
    & 58.11 $\pm$ 6.27 \\
MixHop
    & 86.58 $\pm$ 1.12 & 70.75 $\pm$ 2.95 & 80.75 $\pm$ 2.29
    & 32.22 $\pm$ 2.34 & 73.51 $\pm$ 6.34 & 75.88 $\pm$ 4.90
    & 77.84 $\pm$ 7.73 \\
GCN-Cheby
    & 86.80 $\pm$ 1.08 & 76.25 $\pm$ 1.76 & 88.08 $\pm$ 0.52
    & 36.11 $\pm$ 1.09 & 74.32 $\pm$ 7.46 & 79.41 $\pm$ 4.46
    & 77.30 $\pm$ 4.07 \\
H$_2$GCN
    & 87.71 $\pm$ 1.25 & 76.72 $\pm$ 1.50 & 88.50 $\pm$ 0.64
    & 35.86 $\pm$ 1.03 & 82.16 $\pm$ 4.80 & 86.67 $\pm$ 4.69
    & 84.86 $\pm$ 6.77 \\
ACM-GCN
    & 87.83 $\pm$ 0.95 & 75.56 $\pm$ 1.32 & 89.48 $\pm$ 0.58
    & 35.09 $\pm$ 1.18 & 77.57 $\pm$ 5.26 & 83.53 $\pm$ 3.83
    & 82.70 $\pm$ 6.27 \\
GCNII
    & 86.80 $\pm$ 1.08 & 74.84 $\pm$ 1.48 & 88.27 $\pm$ 0.72
    & 33.48 $\pm$ 2.05 & 54.32 $\pm$ 9.14 & 56.86 $\pm$ 8.32
    & 61.89 $\pm$ 6.43 \\
GPR-GNN
    & 87.42 $\pm$ 1.21 & 75.43 $\pm$ 1.47 & 89.18 $\pm$ 0.51
    & 35.47 $\pm$ 1.66 & 74.32 $\pm$ 3.66 & 79.61 $\pm$ 5.56
    & 74.86 $\pm$ 5.70 \\
GGCN
    & 86.32 $\pm$ 0.91 & 76.65 $\pm$ 1.91 & 88.25 $\pm$ 0.43
    & 34.86 $\pm$ 0.87 & 71.35 $\pm$ 7.34 & 74.12 $\pm$ 5.37
    & 65.14 $\pm$ 8.30 \\
LINKX
    & 77.32 $\pm$ 1.68 & 72.00 $\pm$ 1.90 & 78.39 $\pm$ 1.09
    & 27.06 $\pm$ 1.22 & 39.46 $\pm$ 17.89 & 55.88 $\pm$ 6.29
    & 52.43 $\pm$ 9.21 \\
GloGNN
    & 88.31 $\pm$ 1.15 & 77.41 $\pm$ 1.65 & 89.62 $\pm$ 0.35
    & 37.36 $\pm$ 1.34 & 82.16 $\pm$ 5.82 & 82.35 $\pm$ 5.11
    & 69.19 $\pm$ 11.16 \\
WRGAT
    & 75.47 $\pm$ 2.90 & 76.81 $\pm$ 1.89 & 88.52 $\pm$ 0.92
    & 36.53 $\pm$ 0.77 & 81.62 $\pm$ 3.90 & 86.98 $\pm$ 3.78
    & 83.62 $\pm$ 5.50 \\
DirGNN
    & 86.54 $\pm$ 1.66 & 77.71 $\pm$ 0.78 & 86.94 $\pm$ 0.55
    & 35.76 $\pm$ 6.31 & 76.51 $\pm$ 6.14 & 80.50 $\pm$ 5.50
    & 76.25 $\pm$ 4.68 \\
HiGNN
    & \textbf{89.72 $\pm$ 1.46} & 79.30 $\pm$ 2.13
    & 89.43 $\pm$ 0.53 & 37.21 $\pm$ 1.35
    & 80.00 $\pm$ 4.26 & 85.88 $\pm$ 3.18
    & 86.22 $\pm$ 4.67 \\
BEC-GNN
    & 88.01 $\pm$ 1.56 & 78.14 $\pm$ 1.31
    & 87.00 $\pm$ 0.28 & 34.34 $\pm$ 2.08
    & 65.96 $\pm$ 9.03 & 66.25 $\pm$ 4.55
    & 68.85 $\pm$ 11.02 \\
\midrule
GCN + CASP
    & 88.72 $\pm$ 1.50 & \textbf{79.36 $\pm$ 1.89}
    & 89.68 $\pm$ 0.50 & 37.23 $\pm$ 2.17
    & 59.04 $\pm$ 4.75 & 67.75 $\pm$ 9.82
    & 65.08 $\pm$ 10.00 \\
DirGNN + CASP
    & 87.06 $\pm$ 1.92 & 78.32 $\pm$ 2.89
    & \textbf{90.03 $\pm$ 0.53} & \textbf{40.49 $\pm$ 1.80}
    & \textbf{87.02 $\pm$ 2.77} & \textbf{89.62 $\pm$ 3.21}
    & \textbf{89.51 $\pm$ 4.59} \\
\bottomrule
\end{tabular}%
}
\end{table*}

\subsubsection{Impact of CASP on ECI and Echo Chamber Window}
\label{app:eci_window}

To empirically validate how CASP modulates the echo chamber effect, we construct 
a synthetic SBM graph and artificially induce homophilic and heterophilic 
label structures. Under homophily, node labels are aligned with community 
membership; under heterophily, labels are deliberately misaligned so that 
same-community nodes belong to different classes. Figure~\ref{fig:eci_sbm} 
shows the resulting ECI depth profiles for standard GCN and GCN+CASP.

Under homophily, CASP  amplifies intra-community aggregation, raising 
peak ECI and extending the echo chamber window across all propagation depths. 
This is the intended behavior: a deeper echo chamber compresses within-class 
variance and improves linear separability. Under heterophily, CASP  promotes inter-community mixing, suppressing peak ECI and narrowing the echo 
chamber window relative to standard GCN. This directly reduces the harmful 
region in which same-community nodes of different classes become 
indistinguishable. In both cases, CASP's data-driven calibration of $r$ from 
label structure correctly identifies the beneficial direction and adjusts the 
echo chamber window accordingly, confirming Theorem \ref{thm:task}.

\begin{figure}[htbp]
    \centering
    \includegraphics[width=\linewidth]{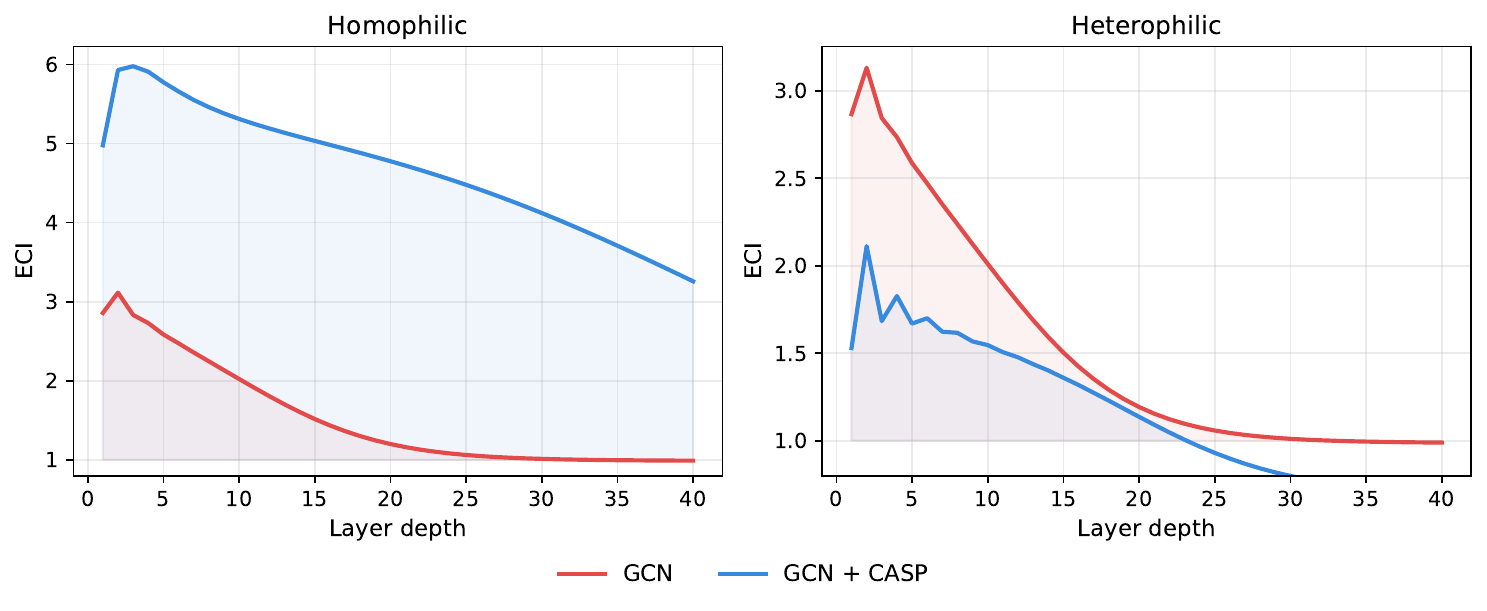}
    \caption{Impact of CASP on ECI and the echo chamber window on a synthetic 
    SBM graph ($n=300$, $m=5$, $p_{\text{intra}}=0.08$, 
    $p_{\text{inter}}=0.005$).}
    \label{fig:eci_sbm}
\end{figure}

\subsubsection{Scalability Analysis}
\label{sec:scalability}

We evaluate CASP on two large-scale OGB benchmarks:
ogbn-arxiv and ogbn-mag~\cite{hu2020open}. Ogbn-arxiv comprises 169,343
nodes, 1,166,243 edges, and 128-dimensional input features. For ogbn-mag,
we follow the OGB evaluation protocol \cite{hu2020open} and use its homogeneous paper-to-paper subgraph,
which contains approximately 736K nodes and 5.4M citation edges.  Table~\ref{tab:ogb-results} reports the large-scale node classification
results. CASP improves GCN on both datasets.

\begin{table}[h]
    \centering
    \caption{Node classification accuracy (\%) on large-scale OGB
    benchmarks.}
    \label{tab:ogb-results}
    \small
    \begin{tabular}{lcc}
        \toprule
        Method & ogbn-arxiv & ogbn-mag \\
        \midrule
        GCN
        & $69.53 \pm 0.28$
        & $30.43 \pm 0.25$ \\
        GCN + CASP
        & $\mathbf{70.12 \pm 0.31}$
        & $\mathbf{31.21 \pm 0.24}$ \\
        \bottomrule
    \end{tabular}
\end{table}

For the scalability analysis on ogbn-arxiv, we use a four-layer GCN
backbone with 256 hidden dimensions. Figure~\ref{fig:scalability}
empirically validates the complexity claims in
Section~\ref{sec:complexity}: the per-epoch runtime and peak activation
memory of CASP-GCN remain constant-factor multiples of those of GCN.
The parameter overhead is negligible because CASP introduces only one
additional scalar parameter. Its two propagation streams are independent
and can also be parallelised on a GPU.

The accuracy improvement confirms that this modest overhead produces a
measurable improvement in representation quality.
Table~\ref{tab:louvain} reports the one-time Louvain community-detection
cost, which is amortised across training runs on the same graph.

\begin{figure}[h]
    \centering
    \includegraphics[width=0.62\linewidth]{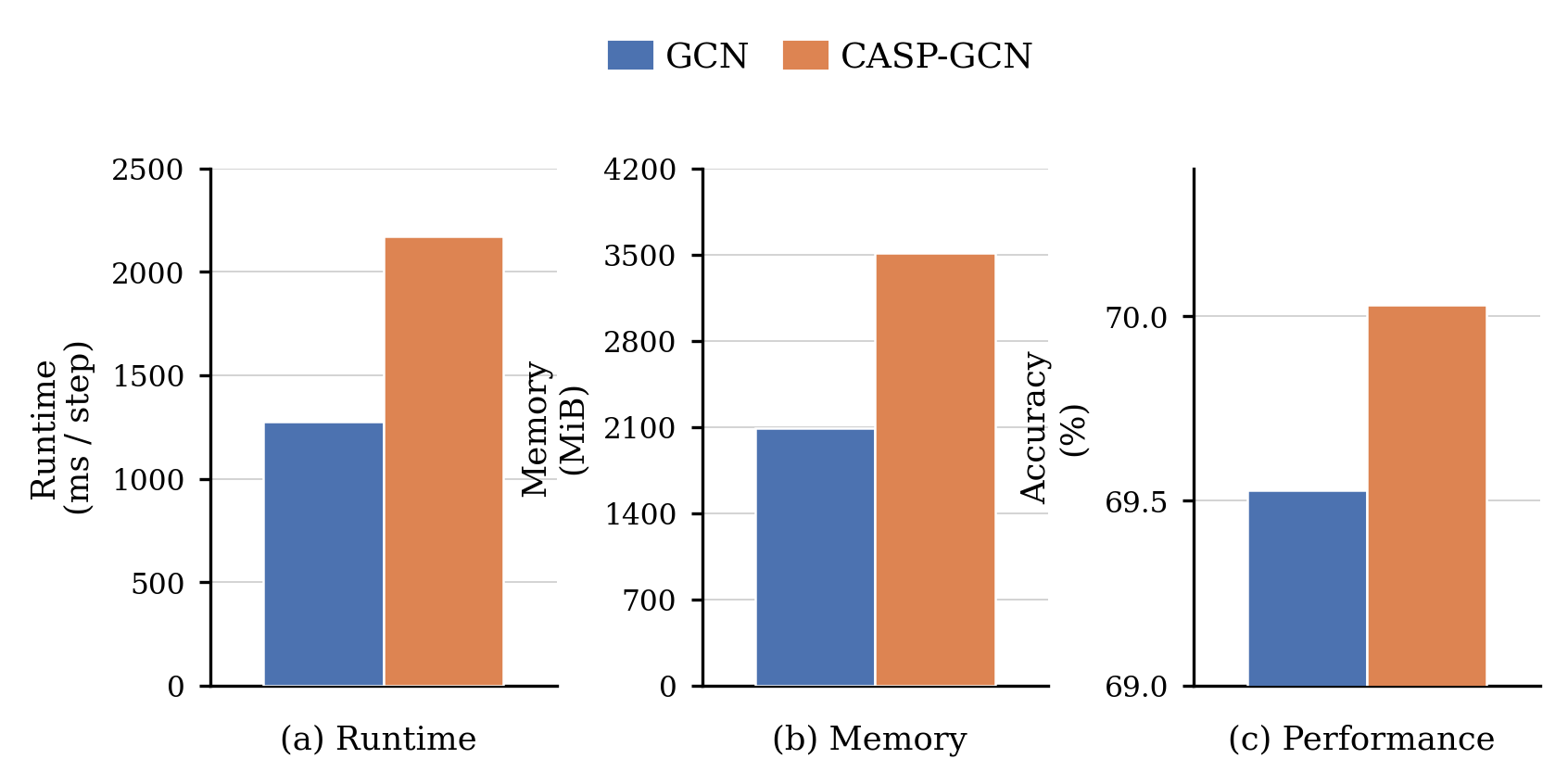}
    \caption{Runtime, memory, and accuracy comparison between GCN and
    CASP-GCN on ogbn-arxiv.}
    \label{fig:scalability}
\end{figure}

\begin{table}[h]
    \centering
    \caption{One-time Louvain preprocessing cost on ogbn-arxiv.}
    \label{tab:louvain}
    \small
    \begin{tabular}{lc}
        \toprule
        Dataset & Time (s) \\
        \midrule
        ogbn-arxiv & 24.69 \\
        \bottomrule
    \end{tabular}
\end{table}

\subsubsection{Sensitivity to Resolution Parameter in Louvain Algorithm}
\label{sec:resolution_ablation}

CASP relies on Louvain community detection to partition edges into intra- and
inter-community sets, from which $p$ and $\delta$ are derived.
The resolution parameter $\gamma$ controls community granularity, where smaller values
yield fewer, coarser communities and larger values produce finer partitions.
We evaluate CASP-DirGNN across $\gamma \in \{0.25, 0.5, 1, 2, 4\}$ on Pubmed
and Amazon-ratings (Table~\ref{tab:resolution}).

Results are stable across all resolutions on Pubmed, with all values
outperforming the baseline by a substantial margin.
On Amazon-ratings, $\gamma = 1$ achieves the highest accuracy.
This is consistent with the role of $\gamma$ in modularity:
$\gamma = 1$ corresponds to the standard objective, which balances
intra- and inter-community connectivity relative to a random-graph
null model~\cite{reichardt2006statistical}.
Deviating from this scale ($\gamma \neq 1$) produces coarser or finer
partitions, which in our experiments weakens the intra/inter signal
and degrades CASP's calibration.
We therefore fix $\gamma = 1$ across all experiments without
dataset-specific tuning.

\begin{table}[h]
\centering
\caption{
    Ablation on Louvain resolution parameter $\gamma$.
    Best result per dataset is highlighted in \textbf{bold}.
    Baseline is DirGNN.
}
\label{tab:resolution}
\small
\setlength{\tabcolsep}{6pt}
\begin{tabular}{lcc}
\toprule
\textbf{Setting} & \textbf{Pubmed} & \textbf{Amazon-ratings} \\
\midrule
Baseline (no CASP)    & $86.94 \pm 0.55$ & $46.66 \pm 0.61$ \\
\midrule
$\gamma = 0.25$       & $\mathbf{90.55 \pm 0.49}$ & $47.66 \pm 0.97$ \\
$\gamma = 0.5$        & $89.99 \pm 0.38$ & $47.21 \pm 1.40$ \\
$\gamma = 1$  & $90.03 \pm 0.53$ & $\mathbf{49.09 \pm 1.19}$ \\
$\gamma = 2$          & $89.97 \pm 0.46$ & $47.36 \pm 1.37$ \\
$\gamma = 4$          & $90.09 \pm 0.47$ & $46.70 \pm 1.45$ \\
\bottomrule
\end{tabular}
\end{table}

%\clearpage
%\input{sections/checklist}

\end{document}